\documentclass[a4paper,11pt,fleqn]{article}

\usepackage{amsmath, amsthm}
\usepackage{amsfonts,stmaryrd}
\usepackage{thmtools} 
\usepackage{thm-restate}
\usepackage{linguex}
\usepackage{txfonts}
\usepackage{mathpazo}
\usepackage{enumitem}
\usepackage[T1]{fontenc}
\usepackage[utf8]{inputenc}
\usepackage[american]{babel}
\usepackage{centernot}
\usepackage{relsize}
\usepackage{exscale}
\usepackage{refcount}
\usepackage{verbatim}
\usepackage{lineno}
\usepackage{nicefrac}
\usepackage[right]{eurosym}
\usepackage{caption}
\usepackage[section]{placeins}
\usepackage{txfonts,pxfonts} %for counterfactual arrows etc. 
\usepackage{pifont} %for the symblos defined below
\newcommand{\cmark}{\ding{51}}%
\newcommand{\xmark}{\ding{55}}
\usepackage[all]{nowidow}
\usepackage[bottom]{footmisc}

\newcommand{\C}{{\sf C}} % Certain Inference
\newcommand{\U}{{\sf U}} % Uncertain Inference
\newcommand{\PP}{{\sf P}} % Possibility-Preserving Inference

\newtheorem{definition}{Definition}
\newtheorem{corollary}{Corollary}

\usepackage{color}

\usepackage{tikz} 
\usetikzlibrary{positioning,arrows, arrows.meta}
\tikzset{ modal/.style={>=stealth',shorten >=1pt,shorten <=1pt,auto,node distance=1.5cm, semithick}, world/.style={circle,draw,minimum size=.5cm}, point/.style={circle,draw,inner sep=0.5mm,fill=black}, reflexive above/.style={->,loop,looseness=7,in=120,out=60}, reflexive below/.style={->,loop,looseness=7,in=240,out=300}, reflexive left/.style={->,loop,looseness=7,in=150,out=210}, reflexive right/.style={->,loop,looseness=7,in=30,out=330}}

\def\half{\nicefrac{1}{2}}
\def\R{\mathbb{R}}
\def\LL{\mathcal{L}}

\newcommand{\amrand}[2]% first argument: author of comments, second argument: comment
  {\leavevmode
   \marginpar
     [\raggedleft\scriptsize \leavevmode\normalcolor #1: #2]
     {\raggedright\scriptsize \leavevmode\normalcolor #1: #2}
  }

\usepackage{hyperref}
\hypersetup{
    colorlinks=true,
    allcolors=black
}

\usepackage[style=philosophy-classic, backend=biber, sortcites=ynt, sorting=nyt, doi=false, isbn=false, url=false, mincitenames=1,maxcitenames=3]{biblatex}
\AtEveryBibitem{%
  \clearfield{number}%only journal volumes, not specific issues
  \clearlist{language}%no languages please
  \clearfield{month}%avoid "Aug. 2007" as the publication date and similar stuff
  \clearfield{day}%see above
  \clearfield{endday}%...
  \clearfield{endmonth}%...
}

\DeclareFieldFormat[book,incollection]{edition}{#1 edition}

\usepackage{setspace}
\usepackage[left=4cm,right=4cm,bottom=4cm,top=4cm]{geometry}
\allowdisplaybreaks[1]
\title{Certain and Uncertain Inference\\with Indicative Conditionals}
\date{} %Draft of \today
\author{Paul \'Egr\'e \and Lorenzo Rossi \and Jan Sprenger}

\begin{document}

%LOR: General (minor) notational issues. Formal theories were indicated either as \sf ({\sf CC/TT}) or as \bf ({\bf Q}) or upright (QCC/SS \cap TT). For now, I uniformed everything as \sf (in line with our previous papers, where we use always \sf). Other minor issue: p is used both as a propositional atom and as the symbol for the probability function (I don't think this is a problem -- just flagging this). Local notational changes are indicated locally.

\maketitle

\begin{abstract}
\noindent 
This paper develops a trivalent semantics for the truth conditions and the probability of the natural language indicative conditional. Our framework rests on trivalent truth conditions first proposed by W. Cooper and yields two logics of conditional reasoning: (i) a logic {\C} of inference from \textit{certain} premises; and (ii) a logic {\U} of inference from \textit{uncertain} premises. But whereas \C\ is monotonic for the conditional, \U\ is not, and whereas \C\ obeys Modus Ponens, \U\ does not without restrictions. We show systematic correspondences between trivalent and probabilistic representations of inferences in either framework, and we use the distinction between the two systems to cast light, in particular, on McGee's puzzle about Modus Ponens. The result is a unified account of the semantics and epistemology of indicative conditionals that can be fruitfully applied to analyzing the validity of conditional inferences. 
\end{abstract}

\sloppy

%\textbf{Word Count:} ca.~12,000 words including footnotes, captions, abstract and references. Not included: equations, figures, tables and appendices
%main text + footnotes + abstract: 1,0000
%references: at least 1,500
%captions and other stuff: 500 

%\pagebreak

\section{Introduction and Overview}\label{sec:intro}

Research on indicative conditionals (henceforth simply ``conditionals'') pursues two major projects: the semantic project of determining their \textit{truth conditions}, and the epistemological and pragmatic project of explaining how we should \textit{reason} with them, and when we can \textit{assert} them. The two projects are related: \textcite[][589]{Jackson1979} states that ``we should hope for a theory which explains the assertion conditions in terms of the truth conditions'' while according to David \textcite[][297]{Lewis1976}, ``assertability goes by subjective probability'', where the value of the latter depends on when a proposition is true or false \parencites[see also][173--174]{Adams1965}[][565]{Jackson1979}[][278]{Leitgeb2017}. 

Ideally, we would have a unified treatment of truth conditions and probability of conditionals and, on that basis, a theory of reasoning with conditionals. Here is the standard approach. Suppose $A$ and $C$ are sentential variables of a propositional language $\LL$ without conditionals, and $\to$ denotes the ``if\ldots then\ldots'' connective. Then, the probability of the sentence $A \to C$ should go by the conditional probability $p(C|A)$ \parencite[e.g.,][]{Adams1965,Adams1975,Stalnaker1970}:
\begin{align*}\label{eqn:TE}
 p(A \to C) &= p(C|A)  \tag{Adams's Thesis}
\end{align*}
The idea is that the conditional ``if the sun is shining, Mary will go for a walk'' seems plausible if and only if it is likely that, given sunshine, Mary is going for a walk.\footnote{The extension of \ref{eqn:TE} to \textit{arbitrary} sentences $A$ and $C$, possibly involving conditionals, is known as ``Stalnaker's Thesis''.} Normative theories of conditionals often recognize \ref{eqn:TE} as a desideratum \parencite[e.g.,][]{Stalnaker1970,Adams1975}. The empirical data are more complicated, but \ref{eqn:TE} is well-supported when the antecedent is \textit{relevant} for the consequent \parencite[e.g., part of the same discourse:][]{OverEtAl2007,SkovgaardOlsenEtAl2016a}. 

Unfortunately, David Lewis's well-known triviality result complicates the picture. \textcite{Lewis1976} showed that, if (i) the probability of a sentence  depends in the standard way on its truth conditions (i.e., as expectation of semantic value),\footnote{In other words, the probability of $A$ corresponds to the total weight of the possible worlds where $A$ is true.} and (ii) the  probability function is closed under conditionalization, Adams's Thesis implies $p(A \to C) = p(C)$, whenever $A$ is compatible with both $C$ and its negation. Similar triviality results have been shown by \textcites{Hajek1989,Bradley2000,Milne2003}. This \textit{reductio ad absurdum} seems to preclude a unified semantic and epistemological treatment of conditionals, at least as far as probability and probabilistic reasoning is concerned.  

We show in this paper that this conclusion is premature: we introduce a third truth value (``neither true nor false'') and propose trivalent truth conditions for natural language indicative conditionals whose probability validates \ref{eqn:TE}. The probabilistic semantics allows us to define a logic for reasoning with \textit{certain} premises as well as a structurally similar logic for reasoning with \textit{uncertain} premises. 

%essential for analyzing controversial inference schemes, and for understanding when they are valid and when they are not. For example, the Or-to-If inference from ``either Alice or Bob will go to the party'' to   will turn out valid in deductive reasoning, but \textit{invalid} in defeasible reasoning. 

In other words, we argue that \textit{different logics of conditionals suit different epistemic situations}. When no conditionals are involved, the epistemic status of the premises does not matter: deductive logic validates all and only those inferences that preserve maximal certainty, i.e., probability one, and also all and only those inferences that do not increase uncertainty \parencite[e.g.,][]{Adams1996Primer}. There is just one notion of valid inference. But conditionals complicate the picture. When premises are supposed as being certain, the inference from ``if Alice goes to the party, Bob will'' to ``if Alice and Carol go to the party, Bob will'' appears valid. Alice's presence \textit{ensures} Bob's presence no matter his feelings for Carol. This picture changes when the premises are taken to be just \textit{likely} instead of certain: Carol going to the party can make a difference if we think that Alice going to the party does not guarantee that Bob will go in all circumstances. Conditional reasoning for uncertain premises has non-monotonic aspects and may require more than one notion of valid inference \parencite[compare][]{Santorio2022PPR}. Our account explains the difference between certain and uncertain reference by keeping the truth conditions of conditionals constant and by building a definition of probability based on those, while relaxing the definition of logical consequence when going from certain to uncertain reasoning. 

%Our account explains the difference between certain and uncertain inference with conditionals systematically with reference to  their truth conditions---and to our knowledge, it is unique in that respect.

% \parencite[e.g.,][]{Hailperin1996}. But when we reason from \textit{uncertain} premises which may be defeated, we need a different, non-monotonic consequence relation. There is no reason to expect a different picture in conditional reasoning: no single logic can be expected to capture inferences with certain and uncertain premises at the same time. As far as we know, our account is the only truth-conditional analysis that explains systematically which conditional inferences are valid in deductive reasoning, and which ones are valid in inductive, defeasible reasoning.}% \parencite[see also][]{Adams1996Preservation}.} 

%we make both deductive inferences, aiming at preserving truth, and defeasible inferences, aiming at preserving justification or assertability. 

We briefly expound the structure of our paper. The first part lays the semantic foundations: Section \ref{sec:basic} motivates the trivalent treatment of conditionals and Section \ref{sec:tables} introduces specific trivalent truth tables for the indicative conditional and the Boolean connectives, giving reasons to select the trivalent conditional operator first introduced by \cite{cooper1968propositional}. Section \ref{sec:ast} defines the (non-classical) probability of trivalent propositions in analogy with defining probability in a conditional-free language.

The second part of the paper focuses on conditional reasoning. From the definition of probability in trivalent semantics, Section \ref{sec:TT} and \ref{sec:LA} derive two logical consequence relations for certainty-preserving inference (=the logic {\C}) and for inferences that do not increase probabilistic uncertainty (=the logic {\U}). We show that {\C} and {\U} can be characterized as preserving semantic values within trivalent logic, and in Section \ref{sec:principles} we examine which principles of conditional logic they validate. In particular, we show that some principles such as Or-to-If or Modus Ponens with nested conditionals are controversisal because they hold in the context of reasoning with certain premises, but fail for uncertain premises.

%JS: added this sentence since the propositions are an important result of the paper
%JS: not necessary in the intro
% The resulting logic {\C} preserves non-falsity and interprets the connectives by means of \citeauthor{cooper1968propositional}'s (\citeyear{cooper1968propositional}) truth tables for the trivalent conditional and negation, conjunction and disjunction. 

The third part contains applications, comparisons and evaluations: Section \ref{sec:MP} discusses nested conditionals and McGee's objection to Modus Ponens from the vantage point of our semantics and the two separate logics for certain and uncertain inference. Section \ref{sec:comp} draws comparisons with other theories. Section \ref{sec:ccl} highlights the strengths and limits of our account. Appendix \ref{app:proofs} provides proof details.

\section{Truth Conditions: The Basic Idea}\label{sec:basic}

It is controversial whether indicative conditionals have factual truth conditions and can be treated as expressing propositions \parencite[e.g., see the dialogue in][]{JeffreyEdgington1991}. According to the non-truth-conditional, probabilistic analysis of conditionals \parencite{Adams1965,Adams1975,Edgington1986,Edgington1995,Edgington2009,over2017defective}, indicative conditionals do not express propositions; at most they have partial truth conditions. 
\begin{quote}\small
  [...] the term `true' has no clear ordinary sense as applied to conditionals, particularly to those whose antecedents prove to be false [...]. In view of the foregoing remarks, it seems to us to be a mistake to analyze the logical properties of conditional statements in terms of their truth conditions. \parencite[][169--170]{Adams1965}
\end{quote}
Non-truth-conditional accounts \textit{stipulate} that $p(A \to C) = p(C|A)$ and develop a probabilistic theory of reasoning with conditionals on the basis of high probability preservation (called ``logic of reasonable inference'' by Adams). This move yields a powerful logic for capturing core phenomena of reasoning with simple conditionals, such as their non-monotonic behavior in certain contexts. This success is recognized by truth-conditional accounts \parencites[e.g.,][485]{McGee1989}[][544]{ciardelli2020indicative}, but the probabilistic approach  severs the link between semantics and epistemology. In particular, it does not cover nested conditionals and compounds of conditionals. Moreover, due to the lack of truth conditions, it does not clarify how one can argue and disagree about conditional sentences in a similar way as we do for normal, non-conditional propositions \parencite[][547]{Bradley2012}.

However, even  a defender of a non-truth-conditional view such as \textcite[][187]{Adams1965} admits that we feel compelled to say that a conditional ``if $A$, then $C$'' has been \textit{verified} if we observe both $A$ and $C$, and \textit{falsified} if we observe $A$ and $\neg C$. For example, take the sentence ``if it rains, the match will be cancelled''; it seems to be true if it rains and the match is in fact cancelled, and false if the match takes place in spite of rain. Indeed, what else could be required for determining the truth or falsity of the sentence?

This ``hindsight problem''  \parencite[the terminology is from][]{khoo2015indicative} is a prima facie reason for treating conditionals as propositions, and assigning them factual truth conditions.  Defenders of non-propositional accounts need to explain why observations in our actual world are \textit{sufficient} for the truth or falsity of ``if $A$, then $C$'', and why $A \to C$ behaves so differently when $A$ is false. 

%LOR: removed "the truth value of", because defenders of non-propositional accounts will not be compelled to say things about the truth value of a conditional, but they should say something on its behavior nonetheless

Truth-conditional accounts of conditionals address this point. They come in various guises: variably strict conditionals \parencite[e.g.,][]{Stalnaker1968}, restrictor semantics \parencite[e.g.,][]{Kratzer2012}, dynamic semantics \parencite[e.g.,][]{Gillies2009}, information state semantics \parencite[e.g.,][]{ciardelli2020indicative,Santorio2022path}, and many more.\footnote{The material conditional analysis, endorsed by Jackson and Lewis, claims that the truth conditions of the indicative and the material conditional agree, and that perceived differences are due to pragmatic, not to semantic factors \parencite{Jackson1979,Grice1989}. This approach, however, gives up on a unified picture of truth conditions and probability in the first place. On that account, if sun were unlikely, the probability of ``if the sun is shining, Mary is going for a walk'' would be close to one regardless of Mary's intentions, which looks unacceptable.} Many of these accounts \textit{emulate} Adams's probabilistic logic of reasonable inference, or central parts thereof. For example, truth preservation in Stalnaker's modal framework famously validates the same inference schemes as Adams's logic in their common domain. All of them, however, face a non-trivial task of modelling the probability of conditionals. Truth preservation works in these logics like a \textit{qualitative} plausibility order, but their analysis of the \textit{quantitative} probability of conditionals must, in the light of Lewis's triviality result, deviate systematically from Adams's thesis. Thus, both the truth-conditional and the non-truth-conditional approaches seem to lose out on some important aspects of conditionals. 

%These are, very roughly, the main roads. Another analysis, endorsed by Jackson and Lewis, claims that the truth conditions of the indicative and the material conditional agree, and that perceived differences are due to pragmatic, not to semantic factors \parencite{Jackson1979,Grice1989}. This approach, however, gives up on a unified picture of truth conditions and probability in the first place. On that account, if sun were unlikely, the probability of ``if the sun is shining, Mary is going for a walk'' would be close to one regardless of Mary's intentions, which looks unacceptable. 

In this paper, we would like to resolve the impasse by treating ``if $A$, then $C$'' as a \textit{conditional assertion}---i.e., as an assertion about $C$ upon the supposition that $A$ is true. Whereas, when the antecedent is false, the speaker is committed to neither truth nor falsity of the consequent. This view takes into account Adams's observation that ``true'' has no clear ordinary sense when applied to indicative conditionals; it has been voiced perhaps most prominently by \textcite[][p.~12, our emphasis]{quine1950methods}: 
\begin{quote}\small
\noindent An affirmation of the form ``if \textit{p} then \textit{q}'' is commonly felt less as an affirmation of a conditional than as a conditional affirmation of the consequent. If, after we have made such an affirmation, the antecedent turns out true, then we consider ourselves committed to the consequent, and are ready to acknowledge error if it proves false. \textit{If on the other hand the antecedent turns out to have
been false, our conditional affirmation is as if it had never been made.} 
\end{quote}
In other words, asserting a conditional makes an epistemic commitment only in case the antecedent turns out to be true. If it turns out to be false, the assertion is retracted: there is no factual basis for evaluating it \parencite[see also][]{belnap1970conditional,belnap1973restricted}. Therefore it is classified as neither true nor false. The ``gappy'' or ``defective'' truth table of Table \ref{tab:CA} interprets this view as a partial assignment of truth values to conditionals  \parencite[e.g.,][]{reichenbach1935wahr,definetti1936logique,Adams1975,baratgin2013uncertainty,over2017defective}.\footnote{Some accounts also take the conditional probability $p(A|C)$ as a possible semantic value for the conditional $A \to C$ \parencite[e.g.,][]{McGee1989,StalnakerJeffrey1994,sanfilippo2020probabilities}, but this analysis reverses the traditional direction of the dependency between the probability and the truth conditions of a sentence: probability should depend on how often we find a sentence to be true, not vice versa.} 

%Sometimes the probability $p(C|A)$ also figures as a surrogate truth value in the lower row of the table \parencites[e.g., in the random variable proposal by][]{StalnakerJeffrey1994}[or in][]{sanfilippo2020probabilities}. 

\begin{table}[htb]
\centering
\begin{tabular}[c]{l|cc}
Truth value of $A \to C$ & $v(C) = 1$ & $v(C) = 0$\\
 \hline
 $v(A) = 1$ & 1 & 0\\
 $v(A) = 0$ & (neither) &  (neither) \\
 \end{tabular}
\caption{\small `Gappy'' or ``defective'' truth table for a conditional $A \to C$ for a (partial) valuation function %LOR: removed, because we have not introduced notations for languages yet. 
%$v: \mathcal{L}^\to \longmapsto \{ 0, 1\}$ 
in a language with conditional.}\label{tab:CA}
\end{table} 

%The resulting account is basically an empiricist view of indicative conditionals, where the truth value of any indicative conditional can in principle be settled by experiment and observation. 

%Conditional assertions $C|A$ are isomorphic to \textit{conditional bets on the consequent}, upon the supposition of the antecedent. Both make a commitment on $C$ that is retracted if $A$ is false. Indeed, a conditional bet on $C$ given $A$ is won if $A$ and $C$ obtain, lost if $A$ and $\neg C$ obtain, and \textit{called off} otherwise. The probability of a conditional assertion can then be defined by the betting odds on the associated conditional bets, i.e., by the ratio between the chances to win and to lose it \parencite[see also][175]{Adams1965}. This approach has been very fruitful as a starting point for a logic of conditionals in uncertain reasoning \parencite[e.g.,][]{Adams1975,Adams1996Primer}: if a conditional $A \to C$ makes a conditional assertion $C|A$, the isomorphism between conditional assertions and conditional bets also carries over to the probability of a conditional. 

However, without a full truth-conditional treatment, such an account is limited: it neither evaluates nested conditionals, nor Boolean compounds of conditionals. If we could  complete Table \ref{tab:CA} and provide full truth conditions in a satisfactory way, this would greatly increase the scope and descriptive power of conditional reasoning, and facilitate the identification of theorems and valid inferences. 

The obvious candidate for such truth conditions is a trivalent truth table, where the absence of commitment to the consequent $C$ is represented by a third truth value. Instead of using partial valuations, we assign a third semantic value, $\half$ or ``indeterminate'', when the antecedent is false (See Table \ref{tab:CAtri}). This is a recurring idea in the literature, defended, among others, by \textcite{definetti1936logique,reichenbach1944philosophic,jeffrey1963indeterminate,cooper1968propositional,belnap1970conditional,belnap1973restricted,manor1975propositional,farrell1986implication,mcdermott1996truth,olkhovikov2016new,cantwell2008logic,rothschild2014capturing,ERS2021JPLa,ERS2021JPLb}.

\begin{table}[htb]
\centering
\begin{tabular}[c]{l|cc}
Truth value of $A \to C$ & $v(C) = 1$ & $v(C) = 1$\\
 \hline
 $v(A) = 1$ & 1 & 0\\
 $v(A) = 0$ & $\half$ & $\half$  \\
 \end{tabular}
\caption{\small Partial trivalent truth table for a conditional $A \to C$ for a partial valuation function 
%LOR: Removed, same reason as above
%$v: \mathcal{L}^\to \longmapsto \{ 0, \half, 1\}$ 
in a language with conditional.}\label{tab:CAtri}
\end{table}

This basic idea has to be developed in various directions. Firstly, we need to decide how to extend the truth table of Table \ref{tab:CA} to a fully trivalent truth table for $A \to C$ where $A$ and $C$ can also take the value $\half$ (=neither true nor false, indeterminate). Secondly, we need to decide how to interpret the standard Boolean connectives $\wedge$, $\vee$, $\neg$ in the context of propositions which can take three different truth values. Doing so will allow us to deal with nested conditionals, and more generally, with arbitrary compounds of atomic sentences connected by the standard connectives and $\to$. Thirdly, we have to define a probability measure for trivalent propositions and a consequence relation for reasoning with certain and uncertain premises. We approach these tasks in turn in the next sections. 

% Thirdly, we need to decide on a \textit{trivalent logical consequence relation} for two kinds of reasoning: for deductive reasoning where premises are certain, and for defeasible reasoning (e.g., Adams's ``reasonable inference'') where premises are uncertain. 

%The latter logic will be our analogue of System {\sf P} \parencite{Adams1975,KLM1990}, Hawthorne and Makinson's System {\sf O} \parencite{Hawthorne1996,HawthorneMakinson2007}, and other probabilistic logics for uncertain reasoning. But first, let us set up the truth tables for the conditional and the Boolean connectives. 

%Since the options for deductive inference have been surveyed in \textcite{ERS2021JPLa,ERS2021JPLb}, we focus in this paper on a logic of uncertain inference {\sf A} that covers conditionals. 

\section{Trivalent Truth Tables}\label{sec:tables}

We start by extending the basic idea of Table \ref{tab:CAtri} to a full trivalent truth table for $A \to C$. The two main options are shown in Table \ref{tab:fincoo} and have been proposed by Bruno \textcite{definetti1936logique} and William \textcite{cooper1968propositional}, respectively. We abbreviate the two connectives with ``DF'' and ``CC'' (the latter after Cooper-Cantwell).\footnote{\textcite{belnap1973restricted}, \textcite{olkhovikov2016new} and \textcite{cantwell2008logic} rediscovered Cooper's truth table independently.} 
%LOR: added sentence about abbreviation. 
In both of them the value $\half$ can be interpreted as ``neither true nor false'', ``void'', or ``indeterminate''. There is moreover a systematic duality between those tables: whereas de Finetti treats indeterminate antecedents like false antecedents, Cooper treats them like true ones. Thus, in de Finetti's table the second row copies the third, whereas in Cooper's table it copies the first. 

\begin{table}[ht]
\centering
\begin{tabular}[t]{l|ccc}
$f_{\to_{DF}}$ & 1 & $\half$ & 0\\
 \hline
 1 & 1 &  $\half$ & 0\\
$\half$ & $\half$ &  $\half$& $\half$\\
 0 & $\half$ &  $\half$& $\half$\\
 \end{tabular}
\qquad
 \begin{tabular}[t]{l|ccc}
$f_{\to_{CC}}$ & 1 & $\half$ & 0\\
 \hline
 1 & 1 &  $\half$ & 0\\
$\half$ & 1& $\half$& 0\\
 0 & $\half$ &  $\half$& $\half$\\
 \end{tabular}
 \caption{\small Truth tables for the de Finetti conditional (left) and the Cooper conditional (right).}\label{tab:fincoo}
\end{table}

Both options can be pursued fruitfully, and the choice between them primarily depends on the results which they yield. Our choice is the Cooper table since it interacts more naturally with  our probabilistic treatment of conditionals and the various notions of logical consequence \parencite[a detailed analysis is given in][]{ERS2021JPLa}. However, for the arguments made in this section, which concern only simple, non-nested conditionals, there is no difference between the two.\footnote{Intermediate options vary the middle row, e.g., with the triple $\langle \half, \half, 0\rangle$ \parencite{farrell1986implication} or the triple $\langle 1, \half, \half\rangle$, suggested by a referee. The former option is reviewed in \cite{ERS2021JPLa}, while the latter option forsakes the equivalence of $\neg (A \rightarrow C)$ and $A \rightarrow \neg C$, typically seen as a desirable property.}
%LOR: Added footnote to address a comment by reviewer 1. 

The second choice concerns the definition of the standard logical connectives. A natural option is given by the familiar \L ukasiewicz/de Finetti/Strong Kleene truth tables, displayed in Table \ref{tab:SK}. Conjunction corresponds to the ``minimum'' of the two values, disjunction to the ``maximum'', and negation to inversion of the semantic value. In particular, the trivalent analysis admits, next to the indicative conditional $A \to C$, a Strong Kleene ``material'' conditional $A \supset C$, definable as $\neg (A \wedge \neg C)$, or equivalently, $\neg A \vee C$.  

\begin{table}[ht]
\centering
\begin{tabular}{c|c}
%\hline
& $f_{\neg}$\\
\hline
$1$ & $0$\\
$\half$ & $\half$\\
$0$ & $1$\\
%\hline
\end{tabular}
\qquad
\begin{tabular}{c|ccc}
%\hline
$f_{\wedge}$ & $1$ & $\half$ & $0$\\
\hline
$1$ & $1$ & $\half$ & $0$\\
$\half$ & $\half$ & $\half$ & $0$\\
$0$ & $0$ & $0$ & $0$\\
%\hline
\end{tabular}
\qquad
\begin{tabular}{c|ccc}
$f_{\vee}$ & 1 & $\half$ & 0\\
 \hline
 1 & 1 & 1 & 1\\
$\half$ & 1 &  $\half$& $\half$\\
 0 & 1 &  $\half$ & 0\\
 \end{tabular}
\caption{\small Strong Kleene truth tables for negation, conjunction, and disjunction.}\label{tab:SK}
\end{table}

The Strong Kleene truth table for negation is uncontroversial and also yields the consequence that the conditional commutes with negation (for either the DF- or the CC-conditional): $\neg (A \to C)$ 
%LOR:  is extensionally equivalent to --> has the same truth table of
has the same truth table as $A \to \neg C$. This is a very natural choice for interpreting conditional assertions: when we argue about $A \to C$, both sides presuppose the antecedent $A$ and argue about whether we should be committed to $C$ or rather to $\neg C$, given $A$ \parencite[see also][247]{Ramsey1929}. 
%the latter constitutes a basis for arguing about $C$. Hence, the negation of the conditional assertion should be the conditional assertion of the negation of the consequent. 
%JS: reformulated the second sentence and deleted the third

Unfortunately, the Strong Kleene truth tables for conjunction and disjunction have a very annoying consequence: ``partitioning sentences'' such as  $(A \to B) \wedge (\neg A \to C)$ will always be indeterminate or false \parencite[][368--370]{belnap1973restricted,Bradley2002}. However, a sentence such as: 
\begin{quote}\label{ex:beach}
\noindent If the sun shines tomorrow, John goes to the beach; and if it rains, he goes to the museum. 
\end{quote} 

\noindent seems to be true (with hindsight) if the sun shines tomorrow and John goes to the beach. This intuition is completely lost in Strong Kleene semantics, regardless of whether we use the de Finetti or the Cooper table for the conditional. Even worse, ``obvious truths'' such as $(A \to A) \wedge (\neg A \to \neg A)$ are always classified as indeterminate.

For this reason, we endorse alternative truth tables for conjunction and disjunction, advocated by \textcite{cooper1968propositional} and \textcite{belnap1973restricted}. See  Table \ref{tab:quasi}. In these truth tables, indeterminate sentences are ``truth-value neutral'' in Boolean operations: true and false sentences do not change truth value when conjoined or disjoined with an indeterminate sentence. This can be motivated by observing that such sentences do not add determinate content as empirical statements do. We call these connectives \textit{quasi-conjunction} and \textit{quasi-disjunction}. They retain the usual properties of Boolean connectives (associativity, commutativity, the de Morgan laws, etc.), solve the problem of partitioning sentences, and have no substantial disadvantages with respect to Strong Kleene truth tables in conditional logic. %LOR: edited footnote, to address Reviewer 2's worry. 
Moreover, they have two non-trivial benefits. 

\begin{table}[htb]
\centering
\begin{tabular}{c|c}
%\hline
& $f_{\neg}$\\
\hline
$1$ & $0$\\
$\half$ & $\half$\\
$0$ & $1$\\
%\hline
\end{tabular}
\qquad
\begin{tabular}{c|ccc}
$f'_{\wedge}$ & $1$ & $\nicefrac{1}{2}$ & $0$\\
\hline
$1$ & $1$ & $1$ & $0$\\
$\nicefrac{1}{2}$ & $1$ & $\nicefrac{1}{2}$ & $0$\\
$0$ & $0$ & $0$ & $0$\\
\end{tabular}
\qquad
\begin{tabular}{c|ccc}
$f'_{\vee}$ & $1$ & $\nicefrac{1}{2}$ & $0$\\
\hline
$1$ & $1$ & $1$ & $1$\\
$\nicefrac{1}{2}$ & $1$ & $\nicefrac{1}{2}$ & $0$\\
$0$ & $1$ & $0$ & $0$\\
\end{tabular}
\caption{\small Truth tables for Strong Kleene negation, paired with quasi-conjunction and quasi-disjunction as defined by \textcite{cooper1968propositional}.}\label{tab:quasi}
\end{table}

First, quasi-disjunction avoids the Linearity principle that $(A\to B) \vee (B\to A)$ cannot be false. This schema was famously criticized by \textcite{maccoll1908if}, who pointed out that neither  ``if John is red-haired, then John is a doctor'', nor ``if John is a doctor, then he is red-haired'', nor their disjunction seems acceptable in ordinary reasoning. A semantics that qualifies such expressions as either true or indeterminate might thus be considered inadequate. 
%\amrand{LOR}{Should we add a probabilistic reading of Chris Scambler's example, to explain away its apparent awkwardness?}
Using quasi-conjunction and quasi-disjunction instead, $(A\to B) \vee (B\to A)$ is false when $A$ is true and $B$ is false (or vice versa).

There is also a principled reason for adopting quasi-conjunction and quasi-disjunction, based on the connection between conditional bets and conditional assertions. How should we evaluate the conjunction of conditional assertions like $(A \to B) \wedge (C \to D)$? The interesting case occurs when $A$ is false, but $C$ and $D$ are true. \textcite[496--501, in particular Theorem 1]{McGee1989}  shows by a Dutch Book argument that in this case, a bet on $(A \to B) \wedge (C \to D)$ should yield a strictly positive partial return. Also \textcite[156]{sanfilippo2020probabilities} argue that we should classify the compound bet as winning. Indeed, to the extent that the sentence $(A \to B) \wedge (C \to D)$ is testable, it has been verified when $A$ is false, but $C$ and $D$ are true. All this suggests to treat the assertion $(A \to B) \wedge (C \to D)$ as true rather than indeterminate. Unlike Strong Kleene conjunction, quasi-conjunction 
%LOR: follows this line of reasoning in this case -->
allows us to model this line of reasoning. 
%LOR: I moved the above sentence here: it was at the end of the section, but it seemed disconnected from the rest of the last paragraph. 
%LOR: We could actually remove this bit on IE, since we say essentially the same in section 8. 

For all these reasons, we adopt quasi-conjunction, quasi-disjunction, Strong Kleene negation and the Cooper truth table for the conditional in the remainder of this paper. %LOR: I specified the language and a characterization of Cooper valuation, since we use them throughout the paper, but there was no characterization of them yet.
Our object-language is the language of propositional logic $\LL$, supplemented with a primitive conditional connective $\rightarrow$, and is indicated as $\LL^\to$. A \emph{Cooper valuation} is a function  $v: \LL^\to \mapsto \{0, \nicefrac{1}{2}, 1\}$ that assigns a semantic value to all sentences of $\LL^\to$ in agreement with the Cooper truth-tables, i.e. it interprets $\neg$ as the strong Kleene negation, $\wedge$ and $\vee$ as Cooper's quasi-conjunction and quasi-disjunction respectively, and $\rightarrow$ as Cooper's conditional. Note finally that all combinations of conditional and conjunctions surveyed in this section validate Import-Export: $(A \wedge B) \to C$ and $A \to (B \to C)$ are extensionally equivalent formulas.
%LOR: added footnote to address remark of referee no. 1. 

\section{Probability for Trivalent Propositions}\label{sec:ast}

Epistemologists capture the standing of a proposition $A$ by the \textit{probability} of $A$, reflecting the agent's evidence for and against $A$. When we identify propositions with sets of possible worlds, the probability of a proposition $A$ is the cumulative credence assigned to all possible worlds where $A$ is true. 

%LOR: mild rephrasing
Trivalent semantics for conditionals implements the same approach using a slight twist. As with bivalent probability, we start with a set of possible worlds $W$ with an associated algebra $\mathcal{A}$, and a weight or credence function $c: \mathcal{A} \to [0,1]$ defined on the measurable space ($W$, $\mathcal{A}$). This function represents the subjective plausibility of a particular element of the algebra, i.e., a set of possible worlds. Our use of possible worlds is devoid of metaphysical baggage and instrumental to define credence functions, as is customary in probabilistic semantics: for us, possible worlds are just Cooper valuations. \footnote{Notably, this does not make the interpretation of the conditional modal or non-truth-funtional: at each world $w$, the truth-value of $A \to C$ is given by a Cooper valuation.} Moreover, we assume that any algebra $\mathcal{A}$ includes the singletons of worlds, i.e., for every $w \in W$, $\{w\} \in \mathcal{A}$. Finally, we assume that the credence function $c$ is finitely additive with $c(\emptyset) = 0$, and $c(W) = 1$. 

%\footnote{The domain of $p$ is restricted to exclude the case in which $A$ takes only value $\nicefrac{1}{2}$ (as in $\bot \to \top$), since in this case the proposition is never true or false, and it is therefore unclear in which sense we could still talk of $A$ having a \textit{probability}, or making an assertion that can be more or less justified. Nevertheless, \emph{counterpossible conditionals} may be judged meaningful and true despite having a necessarily false antecedent. A potentially plausible natural language example is ``if 2+2=5, then 2+3=6''.
%LOR: Minor edits} 
%\amrand{PE}{Is it really correct to talk of odds? I would expect then that we compare $c(X_T)$ to $c(X_F)$ directly for odds---JS: Yes, see changes below} 

We now identify propositions with sentences of $\LL^\to$ and define a (non-classical) probability function $p: \mathcal{L}^\to \longmapsto [0,1]$, taking into account that sentences of $\LL^\to$ can receive three values: true, false, or indeterminate.\footnote{If you do not like to use the term ``probability'' in a non-classical framework, because you prefer to reserve it for standard bivalent probability, just replace it by ``degree of assertability'' or a similar term. This is the choice of \textcite{mcdermott1996truth}, whose definition is identical to ours. Also \textcite{cantwell2006laws} proposes the same definition of trivalent probability on the basis of different truth conditions.} For convenience, define 
\begin{align*}
A_T &= \{w \in W \mid v_w(A) = 1\} & A_I &= \{w \in W \mid v_w(A) = \half\} \\
A_F &= \{w \in W \mid v_w(A) = 0\}
\end{align*}
as the sets of possible worlds where $A$ is valued as true, false or indeterminate, relative to (Cooper) valuation functions $v_w: \LL^\to \mapsto \{0, \half, 1\}$, indexed by the possible worlds they represent. 
%JS: Put the w into the subscript of v_w (instead of v(A, w)): valuation functions have one, not two arguments, and they are isomorphic to possible worlds. 

In analogy to bivalent probability, we derive the probability of a (conditional) proposition $A$ from the {(conditional) betting odds} on $A$: how much more likely is a bet on $A$ to be won than to be lost? For this comparison, two quantities are relevant: (1) the cumulative weight of the worlds where $A$ is true (i.e., $c(A_T)$), and (2) the cumulative weight of the worlds where $A$ is false, i.e., $c(A_F)$). The \textit{decimal odds} on $A$ are $O(A) = (c(A_T)+c(A_F))/c(A_T)$, indicating the factor by which the bettor's stake is multiplied in case $A$ occurs and she wins the bet. Then we calculate the probability of $A$ from the decimal odds on $A$ by the familiar formula $p(A) = 1/O(A)$, yielding
\begin{equation}\label{eqn:ast}
p(A) := \frac{c(A_T)}{c(A_T) + c(A_F)} \qquad \text{if } \max(c(A_T), c(A_F)) > 0. \tag{Probability}
\end{equation}
Hence, the probability of a sentence corresponds to its expected semantic value, \textit{restricted to the worlds where the sentence takes classical truth value}. Additionally, we let $p(A) = 1$ whenever $c(A_T) + c(A_F) = 0$, i.e., if it is certain that $A$ takes the value $\half$ (e.g., when $A$ is $\bot \to \top$). 

In other words, the trivalent probability of $A$ is the ratio between the credence assigned to the worlds where $A$ is true, and the credence assigned to the worlds where $A$ has classical truth value. 
Worlds where $A$ takes indeterminate truth value are neglected for calculating the probability of $A$, except when they take up the whole space. For conditional-free sentences $A$ and their Boolean compounds, this corresponds to the classical picture since $W = A_T \cup A_F$, or equivalently, $A_I = \emptyset$.

The idea behind \eqref{eqn:ast} is the same that motivates classical operational definitions of probability: a proposition is assertable, or probable, \textit{to the degree that we can rationally bet} on it, i.e., to the degree that betting on this proposition will, in the long run, provide us with gains rather than losses  \parencite[e.g.,][]{SprengerHartmann2019}. This is a good reason for calling the object defined by equation \eqref{eqn:ast} a ``probability'', or a measure of the plausibility of a proposition. 

The structural properties of $p: \LL^\to \longmapsto [0,1]$ resemble the standard axioms of probability:
\begin{enumerate}
  \item[(1)] $p(\top) = 1$ and $p(\bot) = 0$. 
  \item[(2)] $p(A) = 1 - p(\neg A)$. 
  \item[(3)] $p(A \vee B) \le p(A) + p(B)$. The equality $p(A \vee B) = p(A) + p(B)$ holds if and only if $A_T \cap B_T = \emptyset$ and $A_I = B_I$.\footnote{The ``only if'' direction presupposes that $p(A) > 0$ and $p(B) > 0$.} 
\end{enumerate}
%While the first and second axioms mirror the standard Kolmogorov axioms, 
Just like standard probability, our trivalent probability is not additive, but \textit{subadditive}. Equality holds here exactly when $A$ and $B$ are incompatible and they take classical truth values in the same set of worlds. The main difference to the standard picture is that the probability of a conjunction can \textit{exceed} the probability of a conjunct. In other words, the ``and-drop'' inference from $X \wedge Y$ to $Y$ will not always preserve probability. 

However, on the betting interpretation of probability, this makes sense: when $A$ and $B$ are false and $C$ is true, the bet on $(A \to B) \wedge C$ yields a positive return, while the bet on $A \to B$ is called off. So we should not expect that in all circumstances $p((A \to B) \wedge C) \le p(A \to B)$, in notable difference to bivalent probability, and some non-classical probability functions \parencite[for a survey, see][]{Williams2016NCL}. Exactly the same phenomenon---the failure of ``and-drop'' in the context of conditional reasoning---was demonstrated in recent experiments by \textcite{SantorioWellwood2023}. Of course, $p(A \wedge B) \le p(A)$ will hold as long as $A$ and $B$ are conditional-free sentences.

On this definition of probability, we obtain for conditional-free sentences $A, C \in \LL$ that 
\begin{align*}%\label{eqn:TE}
p(A \to C) &= %\frac{c(A\to C)_T}{c(A\to C)_T+c(A\to C)_F} = 
\frac{c(A_T \cap C_T)}{c(A_T)} = \frac{p(A \wedge C)}{p(A)} = p(C|A) \tag{Adams's Thesis}
\end{align*}
as for conditional-free sentences, $p(X) = c(X_T)$, and because for bivalent $A$ and $C$, $\frac{c(A\to C)_T}{c(A\to C)_T+c(A\to C)_F}=\frac{c(A_T \cap C_T)}{c(A_T)}$. That is, instead of \textit{postulating} Adams's Thesis as a desideratum on the probability of a conditional, as in \textcite{Stalnaker1970} and \textcite[][3]{Adams1975}, we obtain it immediately from the semantics of trivalent conditionals, and the definition of probability as the inverse of rational betting odds.\footcite[For more discussion of Adams's Thesis, including experimental evidence for and against, see][]{Stalnaker1968,Adams1975,dubois1994conditional,DouvenVerbrugge2010,DouvenVerbrugge2013,EvansEtAl2007,OverEtAl2007,egre2011if,Over2016,SkovgaardOlsenEtAl2016a} The well-known triviality results by \textcite{Lewis1976} and others are blocked since they depend on an application of the (bivalent) Law of Total Probability, which does not hold for trivalent, non-classical probability functions \parencite{Lassiter2020}.\footnote{\textcite{Bradley2000} proposes a different triviality result: arguably we want indicative conditionals to satisfy the Preservation Condition---if $p(A) > 0$ and $p(C) = 0$, then $p(A \to C) = 0$ ---, but for this to hold in full generality, we need to posit strong logical dependencies between a conditional and its components, thus trivializing the conditional. This is indeed so for bivalent accounts, but our trivalent account implies the Preservation Condition as a theorem without having a vicious dependency between the truth values of $A$, $C$ and $A \to C$.} Equipped with a definition of probability, we now proceed to characterizing logical consequence relations for certain and uncertain inference.

\section{Certain Inference}\label{sec:TT}

%\textcolor{olive}{Premises of an inference can be certain or uncertain, and this affects the way we reason with them. Uncertain reasoning with conditionals is arguably non-monotonic: when premises are less than fully certain, it does \textit{not} follow from ``if Alice goes to the party, Bob will'' that ``if Alice and Carol go to the party, Bob will''. Rather, it depends on how strongly Bob dislikes Carol \parencite[][]{Adams1965,Lewis1973}. On the other hand, the inference appears sound if the premise is certain, i.e., Alice's presence ensures Bob's presence. This section explicates the second type of inference.} 

%Mathematicians and logicians accept principles such as Conditional Proof: from $A$ and $B$ it follows that $C$; hence, it follows from $A$ that if $B$, then $C$. 

%Also natural language inferences involving conditionals, such as from ``Either Alice or Bob will go to the party'' to  ``if Alice does not go to the party, Bob will'' appear deductively valid, i.e., truth-preserving. Deductive reasoning is typically assumed to be \textit{monotonic}: adding a premise such as ``Bob is not enthusiastic about the party'' does not compromise the validity of the inference. 

%By contrast, uncertain reasoning with conditionals is essentially \textit{non-monotonic}. 

%How can we account for this difference in a theory of conditionals? 

For a conditional-free propositional language $\LL$ with only two truth values, valid inferences preserve the truth of the premises---or equivalently, they preserve \textit{certainties}  \parencite[i.e., probability one:][]{Leblanc1979}. In a trivalent setting, however, there is no canonical notion of ``truth preservation'': it could amount to preserving strict truth (i.e., semantic value 1), to preserving non-falsities (i.e., semantic value greater than 0), or to a combination of both. It is simply not clear what valid inference amounts to. But there is a canonical extension of \textit{certainty-preserving inference} to $\LL^\to$: whenever all premises have probability one, as defined in the previous section, the conclusion should have probability one, too. We call this logic {\C} like ``inference with certain premises''. Formally:  
%LOR: minor editing
\begin{definition}[Valid Inference in $\C$]
For any set of formulas $\{\Gamma, B\} \subseteq \LL^\to$, the inference from $\Gamma$ to $B$ is $\C$-valid, in symbols $\Gamma \models_{\C} B$, if and only if for all probability functions $p: \mathcal{L}^\to \longmapsto [0,1]$: if $p(A) = 1$ for all $A \in \Gamma$, then also $p(B) = 1$.
\end{definition}
In its spirit, this definition of logical consequence is similar to theories of conditional inference based on preserving acceptability in context  \parencite[e.g.,][]{Gillies2009,Santorio2022path,Santorio2022PPR}---probability 1 is just a specific way of expressing which propositions are \textit{accepted}, and valid inference amounts to preservation of (full) acceptance \parencite[e.g.,][271]{stalnaker1975indicative}. In fact, the properties of $\C$ largely agree with Santorio's preferred system (though not with Gillies's)---but without the limitation to a language involving at most simple conditionals.

Based on the probabilistic characterization of the logic of certain inference {\C}, we can derive which trivalent logic corresponds to it: an inference is \C-valid if and only if non-falsity is preserved in passing from $\Gamma$ to $B$. Equivalently, we cannot assign a designated value ($1$ or $\half$) to the premises without assigning it to the conclusion, too. This is the main result of this section. 
%LOR: minor editing
\begin{restatable}[Trivalent Characterization of \C]{proposition}{propB}
\label{prop:TT}
For any set $\{\Gamma, B\} \subset \LL^\to$, $\Gamma \models_{\C} B$ if and only if for all Cooper valuations $v: \LL^\to \longmapsto \{0, \half, 1\}$:  
\[
\mbox{for every } A \mbox{ in } \Gamma, \mbox{ if } v(A) \ge \half, \mbox{ then } v(B) \ge \half.
\]
%or equivalently, 
%\[
%v(\phi) = 0 \; \Rightarrow \exists X \in \Gamma: v(\psi) = 0. 
%\]
\end{restatable} 
\noindent In other words, {\C} preserves truth in the (weak) sense that we cannot infer a false conclusion from a set of non-false premises. Equivalently, if the conclusion is false, one of the premises must have been false. We have thus established an analogous result to the equivalence between truth-preserving and certainty-preserving inference in standard propositional logic. 

\C{} satisfies principles such as $B \models_{\C} A \to B$, i.e., if we are certain that Bob comes to the party, then we are also certain that Bob comes to the party if Alice does. While this inference is fallacious when premises are uncertain, it is valid in any context where we have \textit{verified} the premise---either empirically or by mathematical proof.\footnote{This behavior of the conditional is similar to the conditional developed in state space semantics, e.g., by \textcite{Leitgeb2017}.} We also have Conditional Proof and other characteristic principles of deductive reasoning in \C, such as Modus Ponens, Modus Tollens and the Law of Identity ($\models_{\C} A \to A$). On the other hand, problematic inferences such as the inference from $\neg A$ to $A \to B$ are blocked. Finally, the laws of classical logic in the conditional-free language $\LL$ (=the Boolean fragment of $\LL^\to$) are also theorems of {\C}, if we restrict ourselves to bivalent valuations. \footnote{{\C} is a paraconsistent logic almost equivalent to Cooper's---his propositional logic of Ordinary Discourse---except that we do not restrict {\C} to bivalent valuations. \textcite{ERS2021JPLa}, who study the entire family of trivalent consequence relations and provide a different argument in favor of {\C}, call it QCC/TT.}

{\C} retains Disjunctive Syllogism ($A \vee B, \neg A \models B$), but gives up Disjunction Introduction ($A \models A \vee B$). However, the counterexample necessarily involves the semantic value $\nicefrac{1}{2}$: when we restrict ourselves to classical \textit{valuations} of atomic sentences, the only invalid instances of $A \models A \vee B$ occur when $A$ is itself a conditional with a false antecedent. This shows that exceptions to the otherwise intuitive rule of Disjunction Introduction addition are quite modest; in fact, \textcite{SantorioWellwood2023} present theoretical and empirical arguments why Disjunction Introduction \textit{should} fail in these circumstances. 

Finally, characterizing {\C} as preserving two designated semantic values ($D = \{1, \half\}$) is not only of theoretical interest, but greatly simplifies the study of this logic: for deciding theorems and valid inferences it suffices to look at the truth tables. Section \ref{sec:principles} studies the theorems and valid inferences in more detail and compares certain inference with \C{} to uncertain inference where instead of certainty, high probability is preserved. Notably, these properties depend on interpreting the conditional using the Cooper truth table: if we had instead paired the de Finetti truth table with preserving non-falsity, we would have lost Modus Ponens---arguably a substantial drawback for a logic that generalizes deductive logic to certain inference with conditionals.  

At this point, the reader may ask what would have happened if we had adopted \textit{strict} truth preservation (i.e., preservation of semantic value 1) as the condition for logical consequence. This logic, let us call it $\models_{\PP}$, preserves strictly positive probability in passing from the premise to the conclusions: 
%the consequence relation of the Cooper SS-logics in the previous section can be represented as preserving strictly positive probability: 
\begin{restatable}[Characterization of Possibility-Preserving Inference]{proposition}{propC}
\label{prop:SS}
Suppose $A, B \in \LL^\to$ and there exists at least one probability function where $p(B) < 1$. Then the following two characterizations of $A \models_{\PP} B$ are equivalent: 
\begin{enumerate}
  \item For all Cooper valuations $v: \LL^\to \longmapsto \{0, \half, 1\}$ such that $v(A) = 1$, it is also the case that $v(B) = 1$.%\footnote{This logic is called QCC/SS in the classification system proposed by \textcite{ERS2021JPLa}.} 
  \item For all credence functions $c: \mathcal{A} \longmapsto \mathbb{R}$ with $c(A_I) < 1$ and associated probability function $p: \LL^\to \longmapsto [0,1]$: if $p(A) > 0$, then $p(B) > 0$.
\end{enumerate}
%JS: GENERALIZATION OMITTED FOR SAKE OF SIMPLE PRESENTATION. More generally, $A_1, A_2, \ldots, A_n \models_QCC/SS B$ if and only if for all credence functions $c: \mathcal{A} \to \mathbb{R}$ with associated probability functions $p: \LL^\to \longmapsto [0,1]$ such that $c((A_i)_I) < 1$ for at least one $1 \le i \le n$, and $c(B_I) < 1$: if $p(\bigwedge A_I) > 0$, then $p(B) > 0$.
%LOR: OLD VERSION $A \models_{\sf CC/SS} B$ if and only if for all probability functions $p: \mathcal{L}^\to \longmapsto [0,1]$: if $p(A) > 0$, then also $p(B) > 0$.
%LOR: SUPERSEDED \amrand{LR}{``weak tautology'' has not yet been explained}
\end{restatable} 
In other words, $A \models_{\PP} B$ if and only if $B$ is a real possibility  (i.e., $p(B) > 0$) in all probability functions that make $A$ a real possibility. While {\C} preserves non-falsity and probabilistic certainties, {\PP} preserves strict truth and probabilistic possibilities \parencite[see also][]{Adams1996Preservation}.\footnote{This logic is called QCC/SS in the classification system proposed by \textcite{ERS2021JPLa}.} Therefore it also satisfies characteristic principles of (conditional) possibility logic, such as the inference from $A \to B$ to $B \to A$. The fact that it satisfies such principles (and fails plausible theorems such as $\models A \to A$) is also a good argument why preservation of (strict) truth is not an adequate consequence relation for reasoning with conditionals. We now move to the main contribution of this paper: developing an account of non-monotonic reasoning with conditionals when premises are uncertain.

\section{Uncertain Inference}\label{sec:LA}

%JS: I killed this para (located after the first sentence of the section) since it is more of interest for logicians than for philosophers. 
%This is an instance of the property of structural monotonicity ($A, B\models A$), paired with the fact that the conditional satisfies conditional introduction. Indeed, the logic {\C} satisfies both properties, and it validates this inference, called True Consequent: $B \models_{\C}  A \to B$.\footnote{Alternatively, we can rephrase this example as a meta-inference from $\Gamma \models B$ to $\Gamma, A \models B$, which is the same as $\Gamma \models A \to B$ since {\C} satisfies Conditional Proof.} 

%LOR: minor rephrasing
Certain inference with conditionals is arguably monotonic: when we know $B$ for certain, or when we suppose it as holding no matter what, we also know that $B$ is the case under the condition that $A$. However, when we move to \textit{uncertain} inference, where only high probability or degree of assertability is preserved, things change. 
%LOR: I've put only "things change", because otherwise we give two distinct characterizations of non-monotonicity. 
We may accept, assert, or find plausible $B$, but reject $B$ under the condition that $A$. For example, the conditional ``if Real Madrid faces Juventus in their next match, then Real Madrid will win'' sounds highly plausible, whereas ``if Real Madrid faces Juventus in their next match but most of their players are sick, then Real Madrid will win'' seems much less plausible. A logic of inference with \textit{uncertain} premises {\U} should therefore, unlike the logic {\C},  be non-monotonic, i.e., we cannot infer from $A \to C$ that $A \wedge B \to C$ for any $A$, $B$ and $C \in \LL^\to$.\footnote{The structural rule of Weakening (that is, inferring $A, B \models C$ from $A \models C$) will remain valid in our logic {\U}. However, the rule $\neg A \models A \rightarrow B$ fails in it, for the same reasons that make it fail in \C.} 

The canonical definition of validity for single-premise inference in a logic of uncertain inference preserves probability, as a proxy for rational acceptance or assertability \parencite[e.g.][]{Adams1975}. In other words, the probability of the premise $A$ must never exceed the probability of the conclusion $B$. Almost all logics of uncertain reasoning agree on this criterion for single-premise inference, which is the natural analogue of truth preservation in certain reasoning. We therefore adopt it as our definition of single-premise logical consequence in uncertain reasoning:

\begin{definition}[Valid Single-Premise Inference in $\U$]
For formulas $A, B \in \LL^\to$: $A\models_{\U} B$ if and only if $p(A) \le p(B)$ for all probability functions $p: \LL^\to \longmapsto [0,1]$ based on credence functions $c: \mathcal{A} \longmapsto \R^{\ge 0}$. 
\end{definition}
\begin{corollary}
$\models_{\U} B$ if and only $p(B) = 1$ for all probability functions $p: \LL^\to \longmapsto [0,1]$ based on credence functions $c: \mathcal{A} \to \R^{\ge 0}$. 
\end{corollary}
\begin{corollary}
{\C} and {\U} have the same theorems. 
\end{corollary}
\noindent It is easy to show that this inference criterion has the following characterization in trivalent logic: 
% \parencite[see also][31]{mcdermott1996truth}
%\footnote{%LOR: Minor edits The only case where no assertability is when both $A$ and $C$ always have value $\nicefrac{1}{2}$ (e.g. they are both of the form $\bot \rightarrow B$). In this case, we simply stipulate that $\bot \rightarrow B \models_{\U} C$ if and only if $p(C) = 1$ for all credence functions $c$ and $A \models_{\U} \bot \rightarrow B$ for all $A$, including the case $\bot \rightarrow B \models_{\U} \bot \rightarrow B$.}   
%JS: This footnote does not convince me. 
%LOR: here and below, rephrased as propositions for uniformity (as these are results and not definitions). 
%LOR: Uniformed with statement in the Appendix. Mild re-phrasing. 
%LOR: More rephrasing might be required. See email. 
\begin{restatable}[Equivalent Characterizations of Valid Single-Premise Inference in {\U}]{proposition}{propD}
\label{prop:SSTT1}
%Suppose $B$ is not a logical validity of {\C}, i.e., $\not \models_{\C} B$. 
For $A, B \in \LL^\to$, the following are equivalent: 
\begin{enumerate}
  \item[(1)] $A\models_{\U} B$
  %For all probability functions $p:\LL^\to \longmapsto [0,1]$, $p(A) \le p(B)$;
  \item[(2)] For all Cooper valuations $v: \LL^\to \longmapsto \{0, \half, 1\}$, $v(A) \le v(B)$, or $\models_{\C} B$. In other words, if $v(A) = 1$ then $v(B) = 1$, and if $v(A) \ge \half$, then $v(B) \ge \half$.
  \item[(3)] $A \models_{\C} B$ and $A \models_{\PP} B$, or $\models_{\C} B$;
  \item[(4)] $A \models_{\C} B$ and $\neg B \models_{\C} \neg A$, or $\models_{\C} B$;
\end{enumerate}
%LOR: OLD VERSION HERE BELOW (SUPERSEDED)
%$A \models_{\sf A} C$ if and only if $\models_{\sf Q} C$, or : 
%  \begin{enumerate}
%    \item[$(1)$] $A \models_{\sf QCC/SS\cap{}TT}  C$; or equivalently,
%    \item[$(2)$] $A \models_{\sf Q} C$ and $\neg C \models_{\sf Q} \neg A$. 
%  \end{enumerate} 
%LOR: SUPERSEDED: Added the extra condition that $C$ is a theorem of {\sf Q} (probably needs checking). 
\end{restatable}
%LOR: Minor edits
%Bracketing the case in which $B$ is a logical truth of \C, 
Condition (2) expresses that the semantic value of the conclusion must not fall below the semantic value of the premise in all possible valuations. By Proposition \ref{prop:TT} and Proposition \ref{prop:SS}, this is equivalent to the conjunction of $A \models_{\C} B$ and $A \models_{\PP} B$ (or $\neg B \models \neg A$), i.e., both certainties and possibilities are preserved.\footnote{\textcite{ERS2021JPLa} call this logic QCC/SS$\cap$TT since it preserves both strict and tolerant truth value (=both strict truths and non-falsities). This is one of the logics entertained in \textcite{belnap1973restricted}.} Thus, {\U} validates fewer inferences than {\C}. The proposition states that all these conditions are equivalent to demanding that the conclusion be at least as probable as the premise for all probability functions. 

%The validity of an inference $A \models_{\U} B$ can thus either be defined via quantifying over probability functions and weight assignments to possible models, or via verifying that the inference preserves certainty as well as possibility. 

%LOR: FOOTNOTE REMOVED BECAUSE SUPERSEDED BY THE APPENDIX. OLD VERSION BELOW: 
%\footnote{Proof sketch: 
%LOR: Minor edits
%If is immediate from Proposition \ref{prop:nodrop} that $C$ is a logical truth of ${\sf Q}$ if and only if $p(C) = 1$, and thus $p(A) \le p(C)$, so we can set this case aside. 
%LOR: This needs checking, but seems correct (just a special case of Proposition 2). Oder?
%$A \models_QCC/TT C$ means that the set of possible worlds where $C$ is false is a subset of the worlds where $A$ is false, and $A \models_QCC/SS C$ means that the set of possible worlds where $A$ is true is a subset of the worlds where $C$ is true. This means that $c(A_T) \le c(C_T)$ and $c(A_F) \ge c(C_F)$, and hence $p(A) \le p(C)$. Conversely, suppose that $C$ is not logical truth of ${\sf Q}$ and that one of the implications is not satisfied, e.g., there is a $w \in W$ such that $v_w(A) = 1$, $v_w(C) \le \half$. Assign maximal (=1-$\epsilon$) weight to this world in order to create a countermodel to $p(A) \le p(C)$.} 

Extending this criterion to multi-premise inference $\Gamma \models B$, for $\Gamma \subseteq \LL^\to$, is non-trivial. Should the probability of $B$ not fall below the minimum probability of the premises? Should it follow Adams's uncertainty preservation criterion \parencite{Adams1975,Adams1996Preservation}? Should $B$ be at least as plausible as the conjunction of the premises? 
%LOR: Intuitively, the choice is not clear and we believe that -->
Since there is no intuitively best candidate here, we believe that the choice should depend on the logical properties of the proposed criterion. We propose that $\Gamma \models_{\U} B$ if and only if for a subset $X \subseteq \Gamma$ of the premises, \textit{the probability of the (quasi-)conjunction of the elements of $X$ never exceeds the probability of the conclusion}, regardless of the choice of the probability function. Formally:
%\begin{description}
  %\item[Validity of Multi-Premise Inference in {\sf A}] $A_1, \ldots A_n \models_{\U} C$ if and only if for all additive functions $m: W \to [0,1]$, 
%  \begin{equation}\label{eqn:MPI}
%    p(A_1 \wedge \ldots \wedge A_n) \le p(B), \tag{Multi-Premise Inference}    
%  \end{equation} 
%where conjunction is interpreted as quasi-conjunction of the premises, as defined in Table \ref{tab:quasi}. 

\begin{definition}[Valid Multi-Premise Inference in $\U$]
For a set of formulas $\Gamma \subseteq \LL^\to$ and a formula $B \in \LL^\to$:
$\Gamma \models_{\U} B$ if and only if there is a finite subset of the premises $\Delta \subseteq \Gamma$ such that for all probability functions $p:\LL^\to \longmapsto [0,1]$,  $p(\bigwedge_{A_i \in \Delta} A_i) \le p(B)$.\label{eqn:MPI}%\tag{Multi-Premise Inference}    
\end{definition}
We define validity by means of existential quantification over (possibly improper) \textit{subsets} of $\Gamma$, in order to preserve the fact that a set of premises entails each of its members, namely $\Gamma \models_{\U} A$ for any $A \in \Gamma$ \parencite[compare][1729]{dubois1994conditional}. %\footnote{\pe{We thank D. Over for encouraging us to preserve this feature of the consequence relation.}}}
%is required for obtaining that a set of premises entails each of his members, i.e., that $\Gamma \models_{\U} A$ for any $A \in \Gamma$ \parencite[compare][1729]{dubois1994conditional}. This property is known as Reflexivity. 
If we required instead that the  quasi-conjunction of \textit{all} members of $\Gamma$ have lower probability than $B$, we would no longer have that $A, B \models_{\U} A$ for every $B$, despite the fact that $A\models_{\U} A$ for every $A$.\footnote{In other words, although the logic would remain reflexive, it would not be structurally monotonic. We are indebted David Over for discussion on this topic.}

%While the conditional in $\U$ is nonmonotonic, we find it simpler to keep to a structurally monotonic consequence relation.

%The condition according to which $\Gamma \vdash \Delta$ whenever $\Gamma$ and $\Delta$ share a sentence is sometimes called Reflexivity, for instance in \cite{scott1971engendering}. However, we prefer to define Reflexivity to be the tighter condition whereby $A\vdash A$ when no context (additional premises or conclusions) intervenes with $A$.}}

%the resulting logic would not be reflexive: as shown in Section \ref{sec:ast}, it is possible that $p(A \wedge B) > p(A)$.

There are also principled reasons for adopting this definition. First of all, Definition \ref{eqn:MPI} allows us to extend the equivalence between probabilistic inference and a trivalent consequence relation from the single-premise to the multi-premise case: 

%\amrand{PE}{Prop above is about probability functions, here we talk of additive functions, why?} CORRECTED
%LOR: turned into a proposition, same reason as above. 
%LOR: More rephrasing might be required. See email. 
\begin{restatable}[Equivalent Characterizations of Valid Multi-Premise Inference in {\U}]{proposition}{propE}
\label{prop:SSTT2}
%Suppose $C$ is not a logical validity of \C, i.e.,$\not\models_{\C} C$. Then 
For $\Gamma \subseteq \LL^\to$ and $B \in \LL^\to$, the following are equivalent:
  \begin{enumerate}
  \item[(1)] $\Gamma \models_{\U}  B$.
    %For all probability functions $p:\LL^\to \longmapsto [0,1]$,    $p(A_1 \wedge \ldots \wedge A_n) \le p(B)$.
       \item[(2)] Either  $\models_{\C} B$, or there is a finite subset of premises $\Delta \subseteq \Gamma$ such that the semantic value of $B$ is, for all Cooper valuations $v$, at least as high as the semantic value of the quasi-conjunction of the premises: $v(\bigwedge_{A_i \in \Delta} A_i) \le v(B)$. 
              % A_i \models_{\rm QCC/SS \cap TT} B$.\footnote{For the TT-logics, but not for the SS-logics, $\bigwedge_{A_i \in X} A_i \models  B$ is the same as $X \models B$. In particular, $A, B \models_{\rm QCC/SS} B$ holds, while $A \wedge B \models_{\rm QCC/SS} B$ does not, due to the choice of quasi-conjunction instead of Strong Kleene conjunction ($v(A) = 1$, $v(B) = \half$).}
    \item[(3)] Either  $\models_{\C} B$, or there is a finite subset of premises $\Delta \subseteq \Gamma$ such that $\bigwedge_{A_i \in \Delta} A_i \models_{\C} B$ and $\bigwedge_{A_i \in \Delta} A_i \models_{\PP} B$. 
    \item[(4)] Either  $\models_{\C} B$, or there is a finite subset of premises $\Delta \subseteq \Gamma$ such that $\bigwedge_{A_i \in \Delta} A_i \models_{\C} B$ and $\neg B \models_{\C} \bigvee_{A_i \in \Delta} \neg A_i$.
\end{enumerate}
\end{restatable}

As for {\C}, the equivalence of (1) with (2), (3) and (4) is not only attractive from a computational point of view, but it also connects probabilistic reasoning with conditionals to the trivalent semantics that defines their truth conditions in the first place. 

Secondly, Proposition \ref{prop:SSTT2} also provides sound and complete calculi for the logic {\U} for free. For instance, since \textcite{cooper1968propositional} has a sound and complete Hilbert-style calculus for {\C}, this automatically translates, thanks to Proposition \ref{prop:SSTT2}, into a sound and complete calculus for {\U}. Validity in {\U} is nothing else but 
%LOR: added the combination of 
the combination of two valid consequence relations in {\C}. Alternatively, still using Proposition \ref{prop:SSTT2}, tableau- and sequent-style sound and complete axiomatizations of {\U} can be extracted from \textcite{ERS2021JPLb}.

Thirdly and finally, defining multi-premise inference in this way yields an attractive set of valid inferences with uncertain premises, as we will see in the next two sections. 

%JS: BLAH BLAH 
%Note also that the logic QCC/SS$\cap$TT that we have rejected in Section \ref{sec:TT} for inference with \textit{certain} premises, e.g. because it does not satisfy Modus Ponens, turns out to be an attractive candidate for inference with \textit{uncertain} premises \parencite[where violating Modus Ponens may be an asset rather than a drawback:][]{McGee1985}.

\section{Properties of \U}\label{sec:principles}

\begin{table}[h!]
\centering
\footnotesize
\begin{tabular}[c]{p{0.42\textwidth}|p{0.48\textwidth}|c|c}
\multicolumn{2}{l|}{\bf Constitutive and Generally Desirable Principles in Uncertain Inference} &  \C & {\U} \\ \hline
Logical Truth & $\models A \to \top$ & \cmark & \cmark \\
Law of Identity & $\models A \to A$ & \cmark & \cmark\\
%LOR: corrected here
Supraclassicality (Laws) & (for $A$ without $\rightarrow$) if $\models_{\sf CL} A$, then $\models A$ & (\cmark) & (\cmark)\\
Left Logical Equivalence & if $A \models_\C B$, $B \models_\C A$, then $A \to C \models B \to C$ & \cmark & \cmark\\
Stronger-Than-Material & $A \to B \models A \supset B$ & (\cmark) & (\cmark)\\
Conjunctive Sufficiency & $A, B \models A \to B$ & \cmark & (\cmark)\\
AND & $A \to B, A \to C \models A \to (B \wedge C)$ & \cmark & \cmark \\
OR & $A \to C, B \to C \models (A  \vee B) \to C$ & \cmark & (\cmark) \\
Cautious Transitivity & $A \to B, (A  \wedge B) \to C \models A \to C$ & \cmark & (\cmark) \\
Cautious Monotonicity & $A \to B, A \to C \models (A  \wedge C) \to B$ & \cmark & \cmark\\
Rational Monotonicity & $A \to B, \neg (A  \to \neg C) \models (A \wedge C) \to B$ & \cmark & \cmark\\
Reciprocity & $A \to B, B \to A \models (A \to C) \equiv (B \to C)$ & \cmark & (\cmark)\\
Right Weakening & if $B \models_Q C$, then $A \to B \models A \to C$ & \cmark & (\cmark) \\
Rule of Conditional K & if $A_1, \ldots, A_n \models_\C C$, then & \cmark & (\cmark)\\ 
& $\qquad (B \to A_1), \ldots, (B \to A_n) \models (B \to C)$ & &\\ \hline
{\bf Optional and Disputed Principles} \\ \hline
Supraclassicality (Inferences) & if $\Gamma \models_{\sf CL} B$ then $\Gamma \models B$ & \xmark & \xmark \\
Modus Ponens & $A \to B, A \models B$ & \cmark & (\cmark) \\
Modus Tollens & $A \to B, \neg B \models \neg A$ & (\cmark) & (\cmark) \\
Simplifying Disjunctive Antecedents & $(A \vee B) \to C \models (A \to C) \wedge (B \to C)$ & (\cmark) & (\cmark) \\ 
Import-Export & $A \to (B \to C)$ if and only if $(A \wedge B) \to C$ & \cmark & \cmark\\
Or-to-If & $\neg A \vee B \models A \to B$ & \cmark & \xmark \\   
Conditional Excluded Middle & $\models (A \to B) \vee (A \to \neg B)$ & \cmark & \cmark\\
\hline
{\bf Connexive Principles (optional)} \\ \hline
Aristotle's Thesis & $\models  \neg (\neg A \to  A)$ & \cmark & \cmark\\
Boethius's Thesis  & $\models (A \to C) \to \neg (A \to \neg C)$ & \cmark & \cmark \\ \hline
{\bf Undesirable Principles} \\ \hline
Contraposition & $A \to C \models \neg C \to \neg A$ & (\cmark) & \xmark \\
Monotonicity & $A \to C \models  (A \wedge B) \to C$ & \cmark & \xmark\\
Transitivity  & $A \to B, B \to C \models A \to C$ & \cmark & \xmark \\ \hline
\end{tabular}  
\caption{\small Overview of Inference Principles involving conditionals in uncertain inference. In the rightmost columns, it is shown whether {\C} and {\U} validate the principle generally (\cmark), only for bivalent valuations of the sentential variables (\cmark in parentheses), or not at all (\xmark).}\label{tab:principles}  
\end{table}

%\amrand{PE}{Extended table with connexive principles, cannot be buried in a fn}

We now evaluate the logic {\U} in terms of the inference schemes it validates, using the principles in Table \ref{tab:principles}, taken from the survey article by \textcite{sep-logic-conditionals}.\footnote{We use {\C} as an appropriate generalization of classical deductive logic in formulating principles like Left Logical Equivalence or Right Weakening.}
%JS: put this sentence into a footnote: it is not central and disturbs the flow
 The principles above the first horizontal line are generally considered to be desirable, or at least not harmful, in uncertain reasoning with conditionals. The principles between the lines---e.g., Modus Ponens, Or-To-If, Import-Export, and Conditional Excluded Middle---are typically a bone of contention between theorists. We also include some tautologies that are distinctive for connexive logics. The principles at the bottom---Contraposition, Monotonicity and Transitivity---are characteristic of most monotonic logics, and logics of deductive inference in particular, but should \textit{not} be satisfied by a non-monotonic logic of uncertain reasoning with conditionals \parencite[for compelling counterexamples, see][]{Adams1965}. So we should expect that these principles are satisfied by {\C}, but not by {\U}. 

Table \ref{tab:principles} evaluates, in the rightmost columns, {\C} and {\U} with respect to all these principles. We cannot discuss each of them in detail, but we make some general observations.
%\footnote{Note that  Conditional Excluded Middle follows from the property that our conditional commutes with negation. That is, $\neg (A \to C)$ has the same truth table as $A \to \neg C$. For indicative conditionals, Conditional Excluded Middle is defended, for example, by \textcite{stalnaker1980defense}.} 
%JS: Deleted this footnote since CEM is only satisfied for AC valuations. 
Many desirable or non-harmful principles are satisfied by {\U} without restriction, whereas some of them only hold for bivalent (``atom-classical'') valuations of at least one sentential variable. This means that when all sentences are conditional-free, the inference is valid; only when one of the sentences contains a conditional connective (so that it can take the third truth value), it is possible that the inference fails. When we compare {\U} to \textit{classical} conditional logics (i.e., logics where all valuations are bivalent, such as Stalnaker-Lewis logics), we can consider the principles valid since making a comparison presupposes bivalent valuations. Specifically, {\U} recovers all valid inferences of System {\sf P}, which is a classical benchmark for conditional logics \parencite{Adams1975,KLM1990}.\footnote{\textcite{Adams1975} characterized his logic of uncertain inference by seven syntactic principles whose combination is known as System {\sf P}: the Law of Identity, AND, OR, Cautious Monotonicity, Left Logical Equivalence, and Right Weakening.} Moreover, both {\C} and {\U} validate connexive principles such as Aristotle's Thesis ($\neg (\neg A \to A)$) and Boethius's Thesis ($(A \to C) \to \neg (A \to \neg C)$).
%JS: streamlined the terminology in this para and below---I always use ``bivalent valuation'' instead of ``atom-classical valuation'', often ``sentential variables'' or ``sentences'' instead of ``variables'' (also, ``conditional-free fragment of L->'', instead of ``AC fragment of L->'')

Principles that are typically considered problematic---Monotonicity, Contraposition, Transitivity, \parencite{sep-logic-conditionals}---are indeed \textit{not} valid in {\U}. These principles do not even hold when we restrict {\U} to bivalent valuations of sentential variables. However, they \textit{do} (mainly) hold in our logic of certain inference \C, in line with our view of {\C} as a generalization of classical deductive logic to a language with a conditional. 

%is \textit{not} due to the introduction of a third truth value, but also holds when all involved propositions are classical (i.e., no proposition contains a conditional connective). 

Most interesting are the six principles in the middle. Supraclassicality fails because {\C} does not support Explosion, e.g., while $A \wedge \neg A \models_{\sf CL} B$ holds for any two sentences $A$ and $B$, it is not the case that $A \wedge \neg A \models_{\C} B$. However, all classical laws are theorems of both {\C} and {\U} when restricted to bivalent valuations. Modus Ponens and Modus Tollens hold for conditional-free sentences, but break down for nested conditionals---in line with McGee's famous objections (see the next section for a detailed analysis). Also Simplification of Disjunctive Antecedent is preserved for bivalent valuations only. 

Import-Export holds unrestrictedly, since $A \to (B \to C)$ and $(A \wedge B) \to C$ 
%are congruent: they (LOR: removed the "congruent" terminology)
have exactly the same truth conditions. The principle is intuitively plausible: ``it appears to be a fact of English usage, confirmed by numerous examples, that we assert, deny, or profess ignorance of a compound conditional $A \to (B \to C)$ under precisely the circumstances under which we assert, deny, or profess ignorance of $(A \wedge B) \to C$'' \parencite[489]{McGee1989}.  Experimental evidence seems to confirm this attitude \parencite{wijnbergen2015probability}. Indeed, the main motivation for giving up Import-Export---e.g., in  Stalnaker-Lewis semantics, but also in the probabilistic semantics of \textcite{sanfilippo2020probabilities}---is not its implausibility, but the pressure from Gibbard's and Lewis's triviality results, where Import-Export is an important premise. Some accounts therefore restrict the validity of Import-Export to simple conditionals and set up an error theory of why we infer from there to the general validity of the principle  \parencite[e.g.,][]{Mandelkern2020IE}. By contrast, both {\C} and {\U} can incorporate Import-Export since the triviality results do not apply to these logics  \parencite{ERS2022Gibbard}. %This strategy is arguably preferable to \textit{ad hoc} solutions and also yields benefits in the analysis of Modus Ponens (see below). 

Conditional Excluded Middle (CEM) is a validity of {\C}, and is therefore valid in {\U} as well. Numerous analyses of indicatives endorse CEM \parencite[e.g.,][]{stalnaker1980defense,Williams2010,ciardelli2020indicative,Santorio2022path}, but there are also notable opponents \parencite[e.g.,][]{Gillies2009,Kratzer2012}. A natural way to argue for CEM is to note that it is an immediate consequence of commutation with negation, i.e., the semantic equivalence between $\neg (A \to B)$ and $A \to \neg B$, which also holds in our system. To see this, note that $(A \to B) \vee \neg (A \to B)$---an instance of the Law of Excluded Middle---immediately entails $(A \to B) \vee (A \to \neg B)$, that is CEM. 

%By the same token, Conditional Non-Contradiction, i.e. $\neg[(A \to B) \wedge (A \to \neg B)]$, immediately follows from $\neg[(A \to B) \wedge \neg (A \to B)]$ (an instance of the Law of Non-Contradiction) via negation commutation. CEM and Conditional Non-Contradiction are therefore immediate, in our setting, from their non-conditional counterparts $A \vee \neg A$ and $\neg (A \wedge \neg A)$ via negation commutation. 

Finally, a crucial difference between {\C} and {\U} concerns the relation of the indicative to the material conditional $A \supset B := \neg A \vee B$ (read as the quasi-disjunction of $A$ and $B$). On the one hand, $A \supset B \models_{\C} A \to B$, i.e., if we know that Alice or Bob ordered a beer, then, if we learn that Alice did not order a beer, we can infer that Bob did so. This apparently valid Or-to-If inference is a classical argument for analyzing the indicative conditional in line with the material conditional, and {\C} captures this intuition. However, this inference is \textit{invalid} when we infer the conditional from an uncertain disjunction. A good illustration of this failure is given by \textcite[p. 191]{Edgington1986}: if I am 90\% confident that it is 8 o'clock, then I am at least as confident that it is 8 or 11 o'clock, but that does not give me the same confidence that if it is not 8 then it is 11 o'clock. Indeed, Or-to-If fails in {\U}, as we want to have it. Actually, neither does the material conditional imply the indicative conditional in {\U}, nor vice versa.  

However, the simple, non-nested indicative conditional often appears to be more demanding to assert than the material conditional \parencite[e.g.,][]{gibbard1980two,Gillies2009}. Can our account then explain this ``Stronger-Than-Material'' intuition? Yes---because for bivalent valuations that use only classical truth values, $A \rightarrow B$ entails $A \supset B$ in both {\C} and in {\U}. In the context of uncertain reasoning with conditional-free statements, $p(A \to B) = p(B|A) \le p(A \supset B)$ is a theorem. In summary, we have Or-to-If as a valid principle for reasoning from certain premises, but not from uncertain premises; nonetheless, we show that why $A \to B$ is less acceptable than $A \supset B$ whenever antecedent and consequent are conditional-free sentences.

\section{Modus Ponens, Tollens, and Import-Export}\label{sec:MP}

Modus Ponens appears invariably valid in inference from \textit{certain} premises, but a famous counterexample by \textcite{McGee1985} challenges its validity in inference from \textit{uncertain} premises. It concerns the 1980 U.S. presidential elections.  
\begin{quote}
\smaller
If a Republican wins the election, then, if Reagan does not win, Anderson will win.

\smallskip

A Republican will win the election.\\
\rule{330pt}{0.2pt}

\smallskip

Therefore, if Reagan does not win the election, Anderson will.
\end{quote}
At some point before the elections, the two premises were commonly accepted: Ronald Reagan was predicted to win the election, and Anderson was the runner-up behind Reagan in the Republicans' primary race. By Modus Ponens we infer that if Reagan does not win, Anderson will. The logical form of that inference is: from $A \to (B \to C)$ and $A$, infer, by Modus Ponens, $B \to C$. However, in the polls  Anderson was actually trailing both Reagan and Carter, the democrat incumbent. Therefore, if Reagan was not elected president, the best prediction would be that Carter would be elected, contradicting the conclusion. 

McGee's counterexample has generated a large amount of literature concerning the validity of Modus Ponens.\footnote{\textcite{SMF1986} respond that the conclusion should be evaluated as a material conditional---which would be a plausible proposition---, and argue that the burden is on McGee to show that this interpretation of the conditional is inadequate. But this defensive strategy is threatened by the strong theoretical and empirical arguments against the material conditional view, in particular the paradoxes of material implication, and the fact that judgments on the probability or assertability of $A \to C$ align with $p(C|A)$, not with $p(\neg A \vee C)$ \parencite[e.g.,][]{OverEtAl2007}.} As stressed by McGee, the intuitive appeal of the counterexample depends crucially on the use of nested conditionals. In particular, \textcite{SternHartmann2018} show that when the major premise of Modus Ponens is a nested conditional, the probability loss in inferring to the conclusion can be much higher than when we apply Modus Ponens to non-nested premises. For bivalent propositions $A$ and $B$, the term
\begin{equation}\label{eqn:MP1}
p(B) = p(B|A) p(A) + p(B|\neg A) (1-p(A))  
\end{equation}
is, by the Law of Total Probability, well controlled by the values of $p(A)$ and $p(B|A)$---the values that represent the probability of the two premises of Modus Ponens. For example, if both values exceed .9, then $p(B) \ge .81$, so the product of the two probabilities is still a reasonably high value. 

However, in the case of right-nested conditionals, the probability of the conclusion of Modus Ponens is poorly controlled: 
%\amrand{LR}{The notation $p(C|A,B)$ and $p(C|\neg A, B)$is not fully clear to me: I've used a conjunction (since we appeal to Import-Export). Probably needs to be checked.} % Paul: I checked it and corrected it Lorenzo, there was a mistake in scope of negation in second term.
\begin{equation}\label{eqn:MP2}
p(C|B) = p(C|A \wedge B) p(A|B) + p(C|\neg A \wedge B) (1-p(A|B))  
\end{equation}
Suppose that  premises are highly plausible, e.g. $p(A) \ge .9$ and $p(C|A \wedge B) \ge .9$, where the latter probability has been calculated by applying Import-Export and Adams's Thesis to $A \to (B \to C)$. Then you can still assign extremely low values to three of the four probabilities on the right hand side of equation \eqref{eqn:MP2}, and derive a very low value of $p(C|B)$. Therefore the probability loss is more pronounced in McGee's example than when we apply Modus Ponens to simple conditionals. 

Our logics mirror this diagnosis: Modus Ponens is valid in {\C}, i.e., in certain inference, and valid in {\U} for 
%LOR: bivalent --> atom-classical, otherwise we cannot infer that B itself is a conditional (with a false antecedent)
bivalent valuations, i.e., when all involved 
%LOR: propositions --> propositional constants, same reason as above
propositional constants are classical. However, {\U} does \textit{not} validate the unrestricted form of Modus Ponens, and in fact, the only countermodel to the schema $A \to B, A \models B$ is $v(A) = 1$ and $v(B) = \half$  (i.e., $B$ is a conditional with false antecedent).\footnote{Suppose that ``A Republican will win'' is true if and only if Reagan or Anderson wins. The main conditional then has  probability 1 (since Or-to-If is valid in \C), the disjunction has high probability, and the consequent has a low probability. Thus, nested Modus Ponens in McGee-type examples fails if and only if the associated Or-to-If inference fails. The fact that McGee's argument is analyzed as valid in \C\ and as invalid in \U\ is also in accordance with the ambivalence generally felt regarding whether the argument is valid or not; specifically, also \textcite{Neth2019} and \textcite{Santorio2022PPR} distinguish between the validity of Modus Ponens in certain and uncertain inference.}
%Applied to McGee's example, the countermodel describes a world where Reagan wins and Anderson doesn't. This valuation assigns value $1$ to the (quasi-)conjunction of the premises (since Reagan is a Republican), but value $\half$ to the conclusion (which is retracted since Reagan did win). 
The same kind of analysis can be applied to showing that Modus Tollens, i.e., the schema $A \to B, \neg B \models \neg A$, is valid for simple conditionals, but not for arbitrary nested conditionals.

Since Import-Export features crucially in McGee's counterexample (e.g., in Stern and Hartmann's probabilistic reconstruction), philosophers and logicians have often faced a choice between both principles. For example, \textcite{Stalnaker1968} and \textcite{Lewis1973} give up Import-Export, but retain Modus Ponens. %, whereas \textcite{McGee1989} makes the opposite choice. 
%Another choice is to retain Modus Ponens, but to restrict the validity of Import-Export, as 
So does \textcite{Mandelkern2020IE}, who restricts the validity of Import-Export.\footnote{More precisely, Mandelkern shows that a conditional satisfying Conditional Introduction (i.e., the meta-inference from $\Gamma, A \models B$ to $\Gamma \models A \to B$) and both Modus Ponens and Import-Export is equivalent to the material conditional. He suggests to restrict  the scope of Import-Export to cases where the ``middle proposition'' $B$ in $A \to (B \to C)$ does not contain a conditional.} Our trivalent framework makes the opposite and arguably more natural choice: like \textcite{McGee1989}, we let Import-Export be unrestrictedly valid and restrict the validity of Modus Ponens. This account does not only give a convincing analysis of McGee-style examples, which are 
%LOR: universally --> typically
typically recognized as a problem for Modus Ponens in uncertain reasoning, but also agrees with psychological evidence in favor of Import-Export and simple Modus Ponens.

\section{Comparisons}\label{sec:comp}

The trivalent treatment of indicative conditionals is first sketched in \textcite{reichenbach1935wahr} and \textcite{definetti1936logique,finetti1936logic}. A more detailed motivation of this approach, including an overview of the main consequence relations of interest, is given by \textcite{belnap1970conditional,belnap1973restricted}, but none of these authors provides a fully worked out account of the logic and epistemology of conditionals. The first complete trivalent account of a logic of conditionals is due to \textcite{cooper1968propositional}, who originally created system {\C}. However, Cooper restricts it to bivalent valuations of the sentential variables, without applying it to the entire language $\LL^\to$, and does not connect it to the probability of conditionals. \textcite{cantwell2008logic} investigates the logical consequence relation of {\C} (=preservation of non-falsity), but uses Strong Kleene connectives for conjunction and disjunction. Moreover, his treatment of ``non-bivalent probability'' ends up with an altogether different probabilistic logic \parencite{cantwell2006laws}.

%Other trivalent approaches like \textcite{goodman1991conditional} and \textcite{dubois1994conditional} pursue an algebraic treatment with a trivalent, or even multivalent evaluation of conditionals \parencite[e.g.,][]{milne1997bruno,Milne2004,sanfilippo2020probabilities}. 

Most similar to our approach, both in spirit and content, are the trivalent accounts developed by \textcite{dubois1994conditional} and \textcite{mcdermott1996truth}. However,  these authors stick to de Finetti's original truth table and (in the case of McDermott) use Strong Kleene truth tables for conjunction and disjunction. The semantic features are thus quite different. On the level of inferences, many features are similar, but McDermott's logic validates Transitivity  ($A \to B$, $B \to C$, therefore $A \to C$). While this is acceptable and even desirable in the framework of \textit{certain} inference, it is arguably problematic when reasoning from \textit{uncertain} premises since the probability of $p(C|A)$ is in no way controlled by $p(C|B)$ and $p(B|A)$; in fact, it can be arbitrarily low. Suppose that you live in a very sunny, dry place. Consider the sentences $A$ = ``it will rain tomorrow'', $B$ = ``I will work from home'', $C$ = ``I will work on the balcony''. Clearly, both $A \to B$ and $B \to C$ are highly plausible, but $A \to C$ isn't. This structural feature offers, in our view, a decisive reason to prefer our model to McDermott's. Dubois and Prade avoid that feature, but like Adams and Cooper, they restrict their account to the flat fragment of $\LL^\to$, i.e., allowing only simple, non-nested conditionals. 

\begin{table}[htb]
\footnotesize
\[
\begin{tabular}[c]{p{0.4\textwidth}||c|c||c|c|c}
\centering
& \multicolumn{2}{|c|}{Trivalent Logics} & \multicolumn{3}{|c}{Bivalent Logics}\\ \hline
Inference Principle & {\U} & {\sf MD} & {\sf P} & {\sf VC} & {\sf C2} \\ 
\hline
%Proponents &  Égré et al. & Adams 1975 & Lewis 1973 & Stalnaker 1968 & McDermott 1996\\
%\hline

Stronger-Than-Material & (\cmark) & \cmark & {\cmark} & \cmark & \cmark \\
Conjunctive Sufficiency &  (\cmark) & \cmark & \cmark & \cmark & \cmark  \\
OR & (\cmark) & \cmark & \cmark & \cmark & \cmark \\
Cautious Transitivity & (\cmark) & \cmark & \cmark & \cmark & \cmark \\
Transitivity  & \xmark & \cmark & \xmark & \xmark &\xmark \\
Modus Ponens & (\cmark) & (\cmark) & \cmark & \cmark & \cmark  \\
Modus Tollens & (\cmark) & (\cmark) & \cmark & \cmark & \cmark \\
Import-Export & \cmark & \cmark & N/A & \xmark & \xmark \\
SDA & (\cmark) & \cmark & N/A %(by def.) 
& \xmark & \xmark  \\  
Rational Monotonicity & \cmark & (\cmark) & {N/A} & \cmark & \cmark  \\
Conditional Excluded Middle & \cmark & \cmark & {N/A} & \xmark & \cmark \\
\end{tabular}
\]
\caption{\small Comparison of the logic {\U} with alternative conditional logics, restricted to inference principles where not all of the logics agree. The surveyed alternatives are System {\sf P}, Lewis' {\sf VC}, Stalnaker's {\sf C2}, and McDermott's {\sf MD}.}\label{tab:comp}  
\end{table}

On the side of reasoning, our logic {\U} generalizes the benchmark account of uncertain reasoning developed in \citeauthor{Adams1975}'s (\citeyear{Adams1975}) monograph \textit{The Logic of Conditionals}. In this book, Adams equates the probability of a conditional $A \to C$ with the conditional probability $p(C|A)$, and develops a probabilistic logic of uncertain reasoning with conditionals on that basis. The descriptive accuracy of the predictions of Adams's logic is acknowledged both by philosophers and by psychologists of reasoning \parencites[e.g.,][487--488]{McGee1989}[544]{ciardelli2020indicative}{OverEtAl2007}{over2017defective}, but due the lack of general truth conditions for compounds and Boolean combinations of conditionals, it has limited scope. The incompleteness of the theory has encouraged more ambitious theorists to pursue different roads (e.g., modal semantics or dynamic semantics). Our account recovers all the inferences in Adams's logic of reasonable inference without suffering from these restrictions. Specifically, some principles that Adams needs to postulate as axioms, such as the equation $p(A \to B) = p(B|A)$ (for $A, B \in \LL$) or the Import-Export Principle, emerge as \textit{corollaries} of our semantics. This makes our account more unified and coherent than Adams's.

We conclude our comparisons with a note on other truth-conditional approaches. The classical modal semantics for a conditional $A \to C$ defines it as true if $C$ is true at the closest possible $A$-world \parencites[e.g., as defined by Stalnaker's selection function or Lewisian spheres:][]{Stalnaker1968,stalnaker1975indicative,Lewis1973,Lewis1973b,McGee1989}. If $A$ is true in the actual world, the truth value of the conditional corresponds to the truth value of the consequent, as in our analysis. The fundamental difference emerges when $A$ is false: while we assign a third truth value to the conditional, modal theorists assign a classical truth value, essentially based on epistemic considerations (``is $C$ the case in a plausible world, or set of worlds, where $A$ is the case?''). In other words, Stalnaker-Lewis semantics creates a disparity between the case where $A$ is true, where truth conditions are factual, and the case where $A$ is false, where truth conditions depend on considerations of plausibility and normality. On our approach, epistemological considerations are relevant for assertion and reasoning, but truth conditions are entirely factual.

%LOR: I've added the remark above in order to address a remark by referee no. 2. I believe this is sufficient, since (as far as the difference we discuss here is concerned) the standard approach and the dynamic one (or the other ones mentioned by the referee) do not diverge much (maybe with the only exception that dynamic semantics can be partial, e.g. for some modals, as in Mandelkern's work). However, what matters here for us is the different treatment of true antecedent vs. false antecedent, and that stands (I believe). I did not explicitly add references because I have troubles with the bib file (my .tex file does not compile references in this format, but I still don't know why), and so I would not be able to make sure I avoided mistakes that would make it not compile. 
%\textcolor{blue}{This disparity stands in need of an explanation, and the trivalent approach presents a more unified semantic and epistemological package}. 

%

%JS: I think the comparison to modal accounts is important: not so much for criticizing them but for highlighting how our approach is different from the mainstream. 

Modern developments of modal semantics go beyond possible-world selection functions. Their common denominator is to evaluate a conditional $A \to C$ as true if $C$ is true in all relevant contexts selected by the antecedent $A$ \parencite[e.g.,][]{Kratzer1986,Mandelkern2019ID}. Specifically, dynamic and information state semantics implement this idea by \textit{updating} on $A$ \parencite[e.g.,][]{Gillies2009,Santorio2022path}. These accounts integrate the semantics of ``if\ldots{}then\ldots'' with the semantics of other modal operators, but they struggle to give a quantitative analysis of the probability of conditional which squares with the truth conditions and yields Adams's thesis \parencite[though see][]{GoldsteinSantorio2021}. The connection to probabilistic reasoning, and the distinction between certain and uncertain inference, is therefore easier to make for us than for them. Moreover, in order to obtain full truth conditions that are stronger than the material conditional, Gibbard's (\citeyear{gibbard1980two}) triviality result forces modal accounts to give up Import-Export (or another very plausible principle such as Supraclassicality), limiting them in their ability to analyze complex conditionals. As explained in Section \ref{sec:principles}, the trivalent account does not need to make similar concessions \parencite[see also][]{Lassiter2020,ERS2022Gibbard}.

\section{Conclusions}\label{sec:ccl}

%The philosophical upshot of the paper is simple: in a trivalent setting the probability of conditionals can be obtained 

%there is no gap between the truth conditions of conditionals and their probabilistic treatment. 

The trivalent analysis in this paper closes the gap between the truth conditions of conditionals, their probabilistic semantics, and our (certain and uncertain) reasoning with them. Specifically, we propose two logics that generalize the concept of valid inference to reasoning with conditionals: {\C} explicates conditional reasoning with certain premises, {\U} explicates conditional reasoning with uncertain premises. Although {\C} is a paraconsistent logic, all theorems of classical logic are also theorems of {\C} when restricted to bivalent valuations. The combination of \C\ and \U\ avoids Gibbard's and Lewis's triviality results, and provides a unified framework for conditional reasoning, in light with the observation that some inference schemes (e.g., Or-To-If, nested Modus Ponens) appear valid in certain and invalid in uncertain reasoning. 

%\footnote{This does not hold for all valid inferences of classical logic since $A \wedge \neg A \not{\models}_{\sf Q} B$, due to the paraconsistent character of {\C}.} 
%JS: footnote deleted, discussed in Section 7 in detail

Summarizing the main features and results of our approach according to topics:

\begin{description}
  \item[Truth Conditions] The indicative conditional expresses a conditional commitment to the consequent, retracted if the antecedent turns out false. This interpretation motivates a fully truth-functional trivalent analysis of the conditional. Following Cooper, we group indeterminate antecedents with true ones, and interpret conjunction and disjunction according to his truth tables for quasi-conjunction and -disjunction. 
  \item[Probability] The probability of a sentence of $\LL^\to$ is the ratio of the weight of possible worlds where it is true, divided by the weight of possible worlds where it is either true or false. Adams's Thesis $p(A \to C) = p(C|A)$ for conditional-free sentences follows as a corollary and need not be postulated as an axiom.
  \item[Certain Inference] Conditional reasoning from \textit{certain} premises is captured by the logic {\C}, which can be characterized as preservation of maximal probability, and equivalently as preservation of non-falsity in trivalent semantics (Proposition \ref{prop:TT}). 
  \item[Uncertain Inference] Conditional reasoning with \textit{uncertain} premises is captured by the logic {\U}, which preserves probability between the quasi-conjunction of the premises and the conclusion. Equivalently, {\U} preserves truth \textit{and} non-falsity for all trivalent valuations of the premises and the conclusion (Proposition \ref{prop:SSTT1} and \ref{prop:SSTT2}).   
\end{description}
Combining these semantic and epistemological elements delivers a coherent and fruitful framework. Specifically, we can use it to analyze and to explain the controversy about the validity of Modus Ponens, Or-to-If, Import-Export and other important inference principles. 

More work needs to be done. The most urgent projects are to explore whether this analysis can in any way be connected to the semantics and epistemology of counterfactuals, and to integrate our analysis with an account of modal operators in natural language, such as ``must'' and ``might''. Possible ways of achieving this are to find an equivalent modal semantics, or to embed the present trivalent approach into a modal framework. We leave these issues for further research.

%Comparisons with competitors are arguably favorable and the practical benefits of a truth-functional account (e.g., automated theorem-proving) are evident. \amrand{PE}{We give ourselves too much of a satisfecit; maybe end with some open problems/ things not considered and to be done} 
%LOR: I restored Jan's original closing sentence, as I thought it was nice. We can remove it again, of course
%We are eager to hear from critics and proponents of competing views. 

%\newpage

\newrefcontext[sorting=nyt] 
\printbibliography

\newpage

\appendix
\section{Proofs of the Propositions}\label{app:proofs}

Given a model, consisting of a nonempty set of worlds $W$ and a valuation function $v$, recall that $A_T, A_I, A_F \subseteq W$ denote the set of possible worlds where $A$ is true, indeterminate, and false, respectively. Here and in the remainder, we identify possible worlds with complete valuation functions to all sentences in the language $\mathcal{L}_\to$. 

%\amrand{PE}{What do we mean by complete valuations? also don't we need to say that $c(\top)>0$?} TAKEN CARE OF

% \propA*

\propB*
%\begin{proposition}\label{prop:CCTT}
%$A \models_{Q} B$ if and only if for all probability functions $p: L_\to \to [0,1]$ that assign $p(A) = 1$, also $p(B) = 1$.
%\end{proposition} 

\begin{proof} 
``$\Rightarrow$''. Suppose $A \models_{\C} B$. This means that for every model, $B_F \subseteq A_F$. Suppose now that $p(A)=1$ for some probability function $p$: by (\ref{eqn:ast}), this requires $c(A_F)= 0$. But since $B_F \subseteq A_F$, and the measure properties of $c$, also $c(B_F) \le c(A_F) = 0$ and hence $p(B)=1$. 

``$\Leftarrow$''. Suppose that for any $p$ with $p(A) = 1$, also $p(B) = 1$. Suppose further that $A \not{\models}_{\C} B$, i.e., there is a model and a world $w \in B_F$ with $w \not{\in} A_F$. Choose 
%LOR: $m$ --> $c$
$c$ such that  $c(w) = 1$, i.e., $w$ has maximal credence, and in particular, $c(w') = 0 \ , \forall w' \ne w$. Then $c(A_F)=c(B_T)=0$, and
\begin{align*}
p(A) &= \frac{c(A_T)}{c(A_T) + c(A_F)} = \frac{c(A_T)}{c(A_T) + 0} = 1,\ \text{but}\\
p(B) &= \frac{c(B_T)}{c(B_T) + c(B_F)} = \frac{0}{0+1} = 0, 
\end{align*}
contradicting what we have assumed. %(If $w \in A_I$, then $c(A_T) = 0$, but the argument does not change.) 
Hence it must be the case that $A \models_{\C} B$. 

The generalization to more than one premise is straightforward since $A_1, \ldots A_n \models_\C B$ if and only if $\bigwedge A_i \models_\C B$. 
\end{proof}
%\amrand{PE}{Why do we need the restriction on B not being Q-valid in statement of Prop? Mention in proof as particular case?} 

\propC* 

%Note: if $\models_\C B$, then $p(A) \le p(B)$ is vacuously true for any $A$ and any probability function $p$ since $p(B)=1$, while $A \models_{\sf QCC/SS} B$ may yet be false. 
%JS: superfluous

%\begin{proposition}\label{prop:CCSS}
%Suppose $\not{\models}_Q B$, i.e., $B$ is no theorem of \textbf{Q} (=there is a possible world where $B$ is false). Then, $A \models_{QCC/SS} B$ if and only if for all probability functions $p: L_\to \to [0,1]$: if $p(A) > 0$, then also $p(B) > 0$.
%\end{proposition} 
\begin{proof} 
``$\Rightarrow$''. Suppose $A \models_{\sf QCC/SS} B$. This means that for every model, $A_T \subseteq B_T$. Suppose now that $p(A)>0$; since we have excluded the case $c(A_I) = 1$ we have strictly positive credence that $A$ is true. In other words, $c(A_T)> 0$. Since $A_T \subseteq B_T$, it follows that $c(B_T) \ge c(A_T) > 0$, and hence $p(B) > 0$.\\
%JS: comments were spot on, but are now superseded \textcolor{blue}{or $c(A_T)+c(A_F)=0$}. \textcolor{blue}{in the latter, $p(A)=1$, but then ?? I think we could have $p(B)=0$. For instance, although $a \models a\vee b$, I think that $c(a_T)$ and $c(a_F)$ could be 0, but $c(a_I)>0$. So $p(a)$ would be 1, but we could have $p(a\vee b)=0$}
``$\Leftarrow$''. Suppose $A_T \ne \emptyset$ (otherwise the proof is trivial). We suppose further that $A \not{\models}_{\sf QCC/SS} B$, i.e., there is a world $w \in A_T$ with $w \not{\in} B_T$. Moreover, by assumption (=$B$ is no theorem of \C) there must be a world $w' \in B_F$. Then we choose $c(w) = c(w') = \nicefrac{1}{2}$ (for the case $w = w'$, choose $c(w) = 1$) and infer 
\begin{align*}
p(A) &= \frac{c(A_T)}{c(A_T) + c(A_F)} = \left( 1 + \frac{c(A_F)}{c(A_T)} \right)^{-1} %\ge  \left( 1 + 2 \cdot c(A_F) \right)^{-1} 
> 0, 
\end{align*}
and moreover, since 
%LOR: $w, w' \ne B_T$ --> $w, w' \ne B_T$, 
$w, w' \notin B_T$, 
\begin{align*}
p(B) &= \frac{c(B_T)}{c(B_T) + c(B_F)} = \frac{0}{0+\nicefrac{1}{2}} = 0, 
\end{align*}
contradicting what we assumed. Hence $A {\models}_{\sf QCC/SS} B$.
\end{proof}

\propD*
%\begin{proposition}\label{prop:nodrop}
%Suppose $B$ is no theorem of \textbf{Q}, i.e., $\not{\models}_Q B$. Then the following are equivalent: 
%\begin{enumerate}
%  \item[(1)] For all probability functions $p: L_\to \to [0,1]$, $p(A) \le p(B)$.
%  \item[(2)] $A \models_Q B$ and $\neg B \models_Q \neg A$. 
%  \item[(3)] $A \models_{QCC/SS\cap TT} B$ . 
%\end{enumerate}
%\end{proposition} 

\begin{proof}
We reason by cases and begin with the case $\models_{\C} B$. In this case, $p(B) = 1$ and hence, (1), (2) and (3) are all true. In the remainder, we can therefore neglect this case and assume that there is at least a world $w \in B_F$. We simplify and unify notation and write ``$\models_{\sf SS}$'' instead of ``$\models_{\sf QCC/SS}$'', and ``$\models_{\sf TT}$'' instead of ``$\models_{\sf QCC/TT}$'' or ``$\models_{\C}$''. First, we show the equivalence of (2) and (3). 
\medskip

(2)$\Rightarrow$(3): By assumption, we already have $A \models_{\sf TT} B$. Suppose $\neg B \models_{\C} \neg A$; this means that $(\neg A)_F \subseteq (\neg B)_F$, or equivalently, $A_T \subseteq B_T$. But the latter is the same as $A \models_{\sf SS} B$. So both the {\sf SS}- and the {\sf TT}-entailment holds between $A$ and $B$.\\

(3)$\Rightarrow$(2): Suppose $A \models_{\sf QCC/SS\cap{}TT} B$. This implies $A \models_{\sf TT} B$ trivially; we still have to show $\neg B \models_{\sf TT} \neg A$. But since $A \models_{\sf QCC/SS} B$, we have $A_T \subseteq B_T$ and hence $(\neg A)_F \subseteq (\neg B)_F$. The latter is equivalent to $\neg B \models_{\sf TT} \neg A$.\\

(3)$\Rightarrow$(1): By assumption, $A_T \subseteq B_T$ and $B_F \subseteq A_F$. Hence, $c(A_T) \le c(B_T)$ and $c(A_F) \ge c(B_F)$. Thus, for all probability functions $p: L_\to \to [0,1]$, 
\begin{align*}
p(A) &= \frac{c(A_T)}{c(A_T)+c(A_F)} = \left(1 + \frac{c(A_F)}{c(A_T)} \right)^{-1} \le \left(1 + \frac{c(B_F)}{c(B_T)} \right)^{-1} = p(B).
\end{align*}
(1)$\Rightarrow$(3): Let us first deal with the case $A_T = \emptyset$. In that case, $A \models_{\sf SS} B$ is trivially satisfied. The only way for (3) to be false is if there is a $w \in A_I \cap B_F$, such that $A \models_{\sf TT} B$ fails. However, in that case, we can assign $c(w) = 1$, obtaining $p(A) = 1$ and $p(B) = 0$. So (1) would fail, too. For this reason, we can presuppose in the remainder that $A_T \ne \emptyset$. 

We now prove the converse, i.e., $\neg$(3) $\Rightarrow \neg$(1). Assume first that $A \not{\models}_{\sf QCC/SS} B$, i.e., $A_T \cap (B_F \cup B_I) \ne \emptyset$.
\begin{description}
  \item[Case 1:] $A_T \cap B_F \ne \emptyset$. Choose a $w \in A_T \cap B_F$ and a probability distribution with $c(w) = 1$, yielding $p(A) = 1$ and $p(B) = 0$. So $\neg$(1) holds.
  \item[Case 2:] $A_T \cap B_F = \emptyset$. Choose a $w \in A_T \cap B_I$. However, since $B$ is by assumption no theorem of QCC/TT, we know that there is a $w' \in B_F$. Assign the credences $c(w) = c(w') = \nicefrac{1}{2}$. Then we obtain the following counterexample to (1):
\begin{align*}
p(A) &= \frac{c(A_T)}{c(A_T) + c(A_F)} \ge \frac{\nicefrac{1}{2}}{\nicefrac{1}{2} + c(A_F)} \ge \nicefrac{1}{2}\\
p(B) &= \frac{c(B_T)}{c(B_T) + c(B_F)} = \frac{0}{0+\nicefrac{1}{2}} = 0. 
\end{align*}
\end{description}
Now assume that $A \not{\models}_{\sf TT} B$, i.e., $B_F \cap (A_T \cup A_I) \ne \emptyset$. If there is a $w \in B_F \cap A_T$, we are done: simply assign maximal credence to this world, and we obtain that $p(A) > p(B)$. If there is only a $w \in B_F \cap A_I$, by contrast, we assign $c(w) = 1/2$, and moreover, we choose an arbitrary $w' \in A_T \cap (B_T \cup B_I)$ with $c(w') = 1/2$. Such a $w'$ must exist since we have assumed $A_T \ne \emptyset$. Then, we construct a counterexample to (1) as follows: 
\begin{align*}
p(A) &= \frac{c(A_T)}{c(A_T) + c(A_F)} = \frac{\nicefrac{1}{2}}{\nicefrac{1}{2} + 0} = 1\\
p(B) &= \frac{c(B_T)}{c(B_T) + c(B_F)} \le \frac{\nicefrac{1}{2}}{\nicefrac{1}{2}+\nicefrac{1}{2}} = \nicefrac{1}{2}
\end{align*}
\end{proof}

The proof of Proposition 4 proceeds exactly as the proof of Proposition 3, with the (quasi-)conjunction $A_1 \wedge \ldots \wedge A_n$ taking the role of $A$. Since there are no structural differences, we omit it. 

\end{document}